\documentclass[12pt,3p]{elsarticle}

\usepackage[utf8]{inputenc}
 \usepackage{csquotes}
 \usepackage{amsmath}
 \usepackage{amssymb}
 \usepackage{setspace}
 \usepackage{graphicx} 
 \usepackage{csquotes}
 \usepackage{mathtools}
 \usepackage{algorithm}
\usepackage{algorithmic}
\usepackage{caption}
\usepackage{enumitem}
\usepackage{amsthm,xspace}
\usepackage{booktabs}
\usepackage{color}
\usepackage{lineno,hyperref}
\modulolinenumbers[5]

\newtheorem{theorem}{Theorem}
\newtheorem{corollary}{Corollary}
\newtheorem{lemma}{Lemma}
\newtheorem{definition}{Definition}

\newcommand{\oneoneea}{(1+1)~EA\xspace}
\newcommand{\oneoneiahype}{(1+1)~IA$^\text{hyp}$\xspace}
\newcommand{\oneoneiahypequal}{{(1+1)~IA}$^\text{hyp}_{\geq}$\xspace}
\newcommand{\oneoneiahypgreater}{{(1+1)~IA}$_>^\text{hyp}$\xspace}

\newcommand{\oneonerlsk}{(1+1)~RLS$_k$\xspace}
\newcommand{\muonerlsageing}{($\mu$+1)~RLS$_p^{\text{ageing}}$\xspace}
\newcommand{\realmuonerlsageing}{($\mu$+1)~RLS$^{\text{ageing}}$\xspace}
\newcommand{\muoneeaageing}{($\mu$+1)~EA$^{\text{ageing}}$\xspace}

\usepackage[symbol]{footmisc}

\journal{Journal of \LaTeX\ Templates}

\begin{document}

\begin{frontmatter}

\title{When Hypermutations and Ageing Enable Artificial Immune Systems to Outperform Evolutionary Algorithms\footnote{An extended abstract of this paper has 
been published at the 2017 Genetic and Evolutionary Computation Conference~\cite{CorusOlivetoYazdani2017}.}}

\author{Dogan Corus}
\ead{d.corus@sheffield.ac.uk}
\author{Pietro S. Oliveto}
\ead{p.oliveto@sheffield.ac.uk}
\author{Donya Yazdani}
\ead{dyazdani1@sheffield.ac.uk}
\address{Rigorous Research, Department of Computer Science, University of Sheffield}
\address{Sheffield, UK}
\address{S1 4DP}




\begin{abstract}
We present a time complexity analysis of the Opt-IA artificial immune system (AIS).
We first highlight the power and limitations of its distinguishing operators (i.e., hypermutations with mutation potential and ageing) by analysing them in isolation.
Recent work has shown that ageing combined with local mutations can help escape local optima on a dynamic optimisation benchmark function.
We generalise this result by rigorously proving that, compared to evolutionary algorithms (EAs), ageing leads to impressive speed-ups
on the standard \textsc{Cliff$_d$} benchmark function both when using local and global mutations.
Unless the {\it stop at first constructive mutation} (FCM) mechanism is applied, we show that hypermutations require exponential
expected runtime to optimise any function with a polynomial number of optima. 
If instead FCM is used, the expected runtime is at most a linear
factor larger than the upper bound achieved for any random local search algorithm using the artificial fitness levels method.
Nevertheless, we prove that algorithms using hypermutations can be considerably faster than EAs at escaping local optima.
An analysis of the complete Opt-IA reveals that it is efficient on the previously considered functions and highlights 
problems where the use of the full algorithm is crucial. We complete the picture by presenting a class of functions for which Opt-IA fails with overwhelming probability while standard EAs are efficient. 
\end{abstract}


\end{frontmatter}


\section{Introduction}

Artificial immune systems (AIS) are a class of bio-inspired computing techniques that take inspiration from the immune system of vertebrates~\cite{decastrobook}.
Burnet's clonal selection theory{\color{black}~\cite{Burnet1959}} has inspired various AIS for function optimisation. 
The most popular ones are Clonalg{\color{black}~\cite{decastro}}, the B-Cell Algorithm \cite{kelsey} and Opt-IA \cite{optia-transaction}.

After numerous successful applications of AIS were reported, a growing body of theoretical work
has gradually been built to shed light on the working principles of AIS. 
While initial work derived conditions that allowed to prove whether 
an AIS converges or not \cite{Oliveto2007a},
nowadays rigorous time complexity analyses of AIS are available. 
Initial runtime analyses focused on studying the performance of typical AIS operators in isolation to explain when and why they are effective.
Such studies have been extensively performed for the contiguous somatic hypermutation operator employed by the B-Cell algorithm \cite{JansenZarges2011a, EasiestFunctions}, 
the inversely proportional hypermutation operator of Clonalg \cite{Zarges2008,Zarges2009} and the ageing 
operator used by Opt-IA \cite{JansenZarges2011c,HorobaJansenZarges09,ageing}.
These studies formed a foundational basis which allowed the subsequent {\color{black}analysis} of the complete B-Cell algorithm as used in practice for
standard combinatorial optimisation~\cite{JansenOlivetoZarges2011,JansenZarges2012}.

Compared to the relatively well understood B-Cell algorithm, the theoretical understanding of other AIS for optimisation is particularly limited.
In this paper we consider the complete Opt-IA algorithm~\cite{optia,optia-transaction}. This algorithm has been shown experimentally to be successful at optimising instances of problems 
such as protein structure prediction~\cite{optia-transaction}, graph colouring~\cite{CutelloGecco03} and hitting set~\cite{cutello2006}.
%
The main distinguishing features of Opt-IA compared to other AIS is their use of an ageing operator and of hypermutations with mutation potentials.
In this work we will first analyse the characteristics of these operators
respectively in isolation 
and afterwards consider a simple, but complete, Opt-IA algorithm. 
The aim is to highlight function characteristics for which Opt-IA and its main components are particularly effective, hence when it may be 
preferable to standard Evolutionary Algorithms (EAs).

The idea behind the ageing operator is that old individuals should have a lower probability of surviving compared to younger ones.
Ageing was originally introduced as a mechanism to maintain diversity. 
Theoretical analyses have strived to justify this initial motivation because the new random individuals (introduced to replace old individuals) typically have very low fitness and die out quickly. 
On the other hand, it is well understood that ageing can be used as a substitute 
for a standard restart strategy if the whole population dies at the same 
generation~\cite{JansenZarges2011c} 
 and to escape local optima if most of the population dies 
except for one survivor that, at the same generation, moves out of the local optimum~\cite{ageing}. 
This effect was shown for a Random Local Search (RLS) algorithm equipped with ageing on the \textsc{Balance} dynamic optimisation benchmark function.
An evolutionary algorithm (EA) using standard bit mutation (SBM)~\cite{OlivetoBookChapter,jansenbook} and ageing would not be able to escape the local optima due to their very large basin of attraction.
Herein, we carefully analyse the ability of ageing to escape local optima on the more general \textsc{Cliff$_d$} benchmark function 
and show that using the operator with both RLS and EAs can make a difference between polynomial and exponential runtimes.

Hypermutation operators are inspired by the high mutation rates occurring in the immune system.
In Opt-IA the {\it mutation potential} is linear in the problem size{\color{black},} and in different algorithmic variants may be static or either increase by a factor that is proportional to the fitness of the solution (i.e., the b-cell)  undergoing the mutation or decrease by a factor that is  inversely proportional to the fitness.
The theoretical understanding of hypermutations with mutation potential is very limited. 
To the best of our knowledge the only runtime analysis available is~\cite{Jansen2011}, where inversely proportional hypermutations were considered, with and without
the {\it stop at first constructive mutation (FCM)} strategy\footnote{An analysis of high mutation rates in the context of population-based evolutionary algorithms was performed
in~\cite{OlivetoLehreNeumann2009}. Increasing the mutation rate
above the standard 1/n value has gained interest in recent years \cite{CorusOlivetoTEVC2018,LissovoiOlivetoWarwickerGECCO2017,DoerrLissovoiOlivetoWarwickerGECCO2017,Lengler2018,FriedrichPPSN2018}.}. 
The analysis revealed that, without FCM, the operator requires exponential runtime to optimise the standard \textsc{OneMax} function,
while by using FCM the algorithm is efficient.
We consider a different hypermutation variant using static mutation potentials and argue that it is just as effective if not superior to other variants. We first show that the use of FCM is essential by rigorously proving that a \oneoneea equipped with hypermutations and no FCM requires exponential expected runtime to optimise any function with a polynomial number of optima. 
We then consider the operator with FCM for any objective function that can be analysed using the artificial fitness level (AFL) method~\cite{OlivetoBookChapter, jansenbook}  and show an upper bound on its runtime that is at most by a linear factor larger than the upper bound obtained for any RLS algorithm using AFL. 
To achieve this, we present a theorem that allows to derive an upper bound on the runtime of sophisticated hypermutation operators by analysing much simpler RLS algorithms with an arbitrary neighbourhood. As a result, all existing results achieved via AFL for RLS may be translated into upper bounds on the runtime of static hypermutations. 
Finally, we use the standard \textsc{Cliff$_d$} and \textsc{Jump$_k$} benchmark functions to show that hypermutations can achieve considerable speed-ups for escaping local optima compared
to well studied EAs.

We then concentrate on the analysis of the complete Opt-IA algorithm.
The standard Opt-IA uses both hypermutations and hypermacromutation (both with 
FCM) mainly because preliminary 
experimental studies for trap functions indicated that this setting led to the best results~\cite{optia,optia-transaction}.
Our analysis reveals that it is unnecessary to use both operators for Opt-IA to be efficient on trap functions. 
To this end, we will consider the simple version using only static 
hypermutations as in~\cite{optia}. We will first consider the algorithm 
with the simplification that we allow genotypic duplicates in the population, to simplify the analysis and enhance the probabilities of ageing to create copies and escape from local optima. Afterwards we extend the analysis to the standard version using a genotype diversity mechanism.
Apart from proving that the algorithm is efficient for the previously considered functions, we present a class of functions called \textsc{HiddenPath}, where
it is necessary to use both ageing and hypermutations in conjunction, hence where the 
use of Opt-IA in its totality is crucial.
Having shown several general settings where Opt-IA is advantageous compared to standard EAs, we conclude the paper by pointing out limitations of the algorithm. In particular, we present a class of functions called \textsc{HyperTrap} that is deceptive for Opt-IA while standard EAs optimise it efficiently {\color{black}with overwhelming probability}.  

Compared to its conference version~\cite{CorusOlivetoYazdani2017}, this paper has been improved in several ways. Firstly, we have extended our analyses of the ageing operator and Opt-IA to include the genotype diversity mechanism as in the algorithm proposed in the literature~\cite{optia,optia-transaction}. Another addition is the introduction of a class of functions where Opt-IA fails to find the optimum efficiently, allowing us to complete the picture by highlighting problem characteristics where Opt-IA succeeds and where it does not. Finally, this paper includes some proofs which were omitted from the conference version due to page limitations.

The rest of the paper is structured as follows. In Section \ref{preliminaries}, we introduce and define Opt-IA and its operators. In Section \ref{sec:hyper}, we present the results of our analyses of the static hypermutation operator in a simple framework to shed light on its power and limitations in isolation. In Section \ref{sec:ageing}, we present our analyses of the ageing operator in isolation and highlight its ability to escape from local optima. In Section \ref{sec:optia}, we present the results of our analyses of the complete algorithm. In Section \ref{sec:diversity}, we extend the analyses to include the genotype diversity mechanism as applied in the standard Opt-IA~\cite{optia,optia-transaction}. Finally, we conclude the paper with a discussion of the results and directions for future work.


\section{Preliminaries} \label{preliminaries}
 \begin{algorithm}[t]
 \begin{algorithmic}[1]
 \STATE{$t=0$,}
 \STATE{{\color{black}initialise $P^{(t)}=\{x_1,...,x_\mu\}$, a population of 
$\mu$ b-cells uniformly at random and set $x_i^{age}=0$ for $i=\{1,...\mu\}$}.}
 \WHILE {termination condition is not reached}
 \STATE{$P^{(clo)}$=~Cloning~($P^{(t)}, dup$),}
 \STATE{$P^{(hyp)}$=~Hypermutation~($P^{(clo)}, c)$,}
 \STATE{$P^{(macro)}$=~Hypermacromutation~($P^{(clo)})$,}
 \STATE{Ageing~$(P^{(t)},P^{(hyp)} \cup P^{(macro)}, \tau, \mu)$,}
 \STATE{$P^{(t+1)}=~\text{Selection}~(P^{(t)},P^{(hyp)} \cup P^{(macro)}, \mu,1)$.}

\STATE{$t=t+1$.}
 \ENDWHILE
  \end{algorithmic}
  \caption{{\color{black}{Opt-IA$^{*}$}}~\cite{optia-transaction}.  {\color{black} Subroutines are described in Algorithms \ref{alg:cloning}, \ref{alg:static}, \ref{alg:h-ageing} and \ref{alg:selection}.}}
 \label{optia-frame}
 \end{algorithm}
{\color{black}
In this section we first present the standard Opt-IA as applied in~\cite{optia-transaction} (called Opt-IA$^*$ from now on) for the maximisation of $f:\{0,1\}^n\rightarrow 
\mathbb{R}$ and then a slightly different version which we will analyse.}

The {\color{black}Opt-IA$^*$} pseudo-code is given in Algorithm \ref{optia-frame}.
It is initialised with a {\color{black}{\mbox{population}}} of $\mu$ b-cells, {\color{black}representing candidate solutions,} generated uniformly at random with $\mathit{age=0}$. 
In each generation, the algorithm creates a new parent population consisting of {\it dup} copies of each b-cell (i.e., Cloning)
which will be the subject of variation. The pseudo-code of the Cloning operator is given in Algorithm \ref{alg:cloning}.

\begin{algorithm}[t]
\begin{algorithmic}[1]
\STATE{$P^{(clo)}=\emptyset$.}
\FOR{all $x_i \in P^{(t)}$}
\STATE{copy $x_i$ $dup$ times,}
\STATE{add the copies to $P^{(clo)}$.}
\ENDFOR
\end{algorithmic}
\caption{Cloning~$(P^{(t)}, dup)$}
\label{alg:cloning}
\end{algorithm}

The variation stage in {\color{black}Opt-IA$^*$} uses a {\it hypermutation operator with mutation 
potential} sometimes followed by {\it hypermacromutation}~\cite{optia-transaction}, sometimes not~\cite{optia}. The {\color{black}Hypermacromutation operator is essentially the same as the well-studied contiguous somatic mutation operator of the B-Cell algorithm; it chooses two integers $i$ and $j$ at random such that $(i+1)\leq j\leq n$, then mutates at most $j-i+1$ values in the range of $[i,j]$.}
If both operators are applied, they act on the clone population (i.e., not in sequence) such that they generate $\mu$ mutants each.
The number of bits $M$  that are flipped by the hypermutation operator is determined by a function called {\it mutation potential}. Three different potentials have been considered in the literature:
{\it static}, where the number of bits that are flipped is linear in the problem size and does not depend on the fitness function\footnote{In~\cite{optia} the mutation potential is declared to be a constant $0<c \leq 1$. This is obviously a typo: the authors intended the mutation potential to be $cn$, 
where $0<c \leq 1$.}, 
{\it fitness proportional} (i.e., a linear number 
of bits are always flipped but increasing proportionally with the fitness of the mutated b-cell) 
and {\it inversely fitness proportional}. The latter potential was previously theoretically analysed in~\cite{Jansen2011}.
What is unclear from the literature is whether the $M$ bits to be flipped should be distinct or not and, when using FCM, whether a {\it constructive mutation} is a strictly improving move 
or whether a solution of equal fitness suffices.
In this paper we will consider the static hypermutation operator with pseudo-code given in Algorithm~\ref{alg:static}. 
In particular, the $M$ flipped bits will always be distinct and both kinds of 
constructive mutation{\color{black}s} will be considered.
At the end of the variation stage all created individuals have $age=0$ if their fitness is higher than that of their parent cell, otherwise they
inherit their parent's age. Then the whole population (i.e., parents and offspring) undergoes the ageing process in which the age of each b-cell is increased by one.
Additionally, the ageing operator removes old individuals. Three methods have been proposed in the literature for such an operator: {\it static ageing}, which deterministically removes all individuals
who exceed age $\tau$; {\it stochastic ageing}, which removes each individual at each generation with probability $p_{die}$; 
and the recently introduced {\it hybrid ageing}~\cite{ageing}, where individuals have a probability $p_{die}$ of dying only once they reach an age of $\tau$.
In~\cite{ageing} it was shown that the hybrid version allows to escape local optima, hence we employ this version in this paper and give its pseudo-code in  Algorithm~\ref{alg:h-ageing}. 

The generation ends with a selection phase {\color{black}for which the 
pseudo-code is given in Algorithm~\ref{alg:selection}}. 
If the total number of b-cells that have survived the ageing operator is larger than $\mu$, 
then a standard $(\mu+\lambda)$ selection scheme is used with the exception that
genotype duplicates are not allowed. 
If the population size is less than $\mu$, then a birth phase fills the population up to size $\mu$ by introducing random b-cells of $age=0$.

In this paper we give evidence that disallowing genotypic duplicates may be detrimental because, as we will show, genotypic copies may help the ageing operator to escape local optima more efficiently. {\color{black}Considering the results of the investigations made in Sections~\ref{sec:hyper}~and~\ref{sec:ageing},
the Opt-IA we will analyse does not apply hypermacromutation and uses static hypermutation coupled with FCM as variation operator, hybrid ageing 
and  standard $(\mu+\lambda)$ selection which allows genotypic duplicates. The pseudo-code is given in Algorithm~\ref{modified-optia} for clarity.}

\begin{algorithm}[t]
\begin{algorithmic}[1]
\STATE {$M=cn$,}
\STATE {$P^{(hyp)}=\emptyset$.}
\FOR{all $x_i \in P^{(clo)}$}
\IF{FCM is not used}
\STATE{create $y_i$ by flipping $M$ distinct bits selected uniformly at random.}
\ELSE
\STATE{create $y_i$ by flipping at most $M$ distinct bits selected uniformly at random one after another
until a \textit{constructive mutation} happens.}
\ENDIF
\IF{$f(y_i) > f(x_i)$}
\STATE{$y_i^{age}=0$.}
\ELSE
\STATE{$y_i^{age}=x_i^{age}$.}
\ENDIF
\STATE{add $y_i$ to $P^{(hyp)}$.}
\ENDFOR
\end{algorithmic}
\caption{Static hypermutation~($P^{(clo)}, c)$}
\label{alg:static}
\end{algorithm}


\begin{algorithm}[t]
\begin{algorithmic}[1]
\FOR{all $x_i \in (P^{(t)} \cup P^{(hyp)})$}
\STATE{$x_i^{age}=x_i^{age}+1$,}
\IF{$x_i^{age}>\tau$ }
\STATE{remove $x_i$ with probability $p_{die}=1-1/\mu$.}
\ENDIF
\ENDFOR
\end{algorithmic}
 \caption{Hybrid ageing~$(P^{(t)},P^{(hyp)}, \tau, \mu)$}
  \label{alg:h-ageing}
\end{algorithm}

\begin{algorithm}[t]
\begin{algorithmic}[1]

\STATE{$P^{(t+1)}= \left(P^{(t)}\cup P^{(hyp)}\right) $.}
\IF{{\color{black}$div$=1}}  
\STATE{remove any offspring with the same genotype as individuals in $P^{(t)}$.}
\ENDIF
\IF{$|P^{(t+1)}|> \mu$}
\STATE{remove the $\left(|P^{(t+1)}|-\mu \right)$ individuals with the lowest fitness breaking ties uniformly at random.}
\ENDIF
\IF{$|P^{(t+1)}|< \mu$}
\STATE{add $\left(\mu-|P^{(t+1)}|\right)$ individuals initialised uniformly at random.}
\ENDIF

\end{algorithmic}
 \caption{Selection~$(P^{(t)},P^{(hyp)}, \mu, div)$}
  \label{alg:selection}
\end{algorithm}

%
%
%


\begin{algorithm}[t]
\begin{algorithmic}[1]
\STATE{$t=0$,} 
\STATE{{\color{black}initialise $P^{(t)}=\{x_1,...,x_\mu\}$, a population of $\mu$ b-cells uniformly at random and set $x_i^{age}=0$ for $i=\{1,...\mu\}$}.}
\WHILE {the optimum is not found}
\STATE{
$P^{(clo)}$=~Cloning~($P^{(t)}, dup$),}
\STATE{
$P^{(hyp)}$=~Static hypermutation~($P^{clo}, c)$,}
\STATE{
Hybrid ageing~$(P^{(t)},P^{(hyp)}, \mu, \tau)$,}
\STATE{Selection~$(P^{(t)},P^{(hyp)}, \mu, div)$, 
}
\STATE{$t=t+1$.}
\ENDWHILE
  \end{algorithmic}
  \caption{{\color{black}Opt-IA}}
 \label{modified-optia}
\end{algorithm}

\section{Static Hypermutation} \label{sec:hyper} 
The aim of this section is to highlight the power and limitations of the static hypermutation operator in isolation.
For this purpose we embed the operator into a minimal AIS framework
that uses a population of only one b-cell and creates exactly one clone per generation. 
The resulting \oneoneiahype, depicted in Algorithm \ref{IAhype}, is essentially a \oneoneea 
that applies the static hypermutation operator instead of using SBM. 
We will first show that, without the use of FCM, hypermutations are inefficient variation operators  for virtually any optimisation function of \mbox{interest}.
From there on we will only consider the operator equipped with FCM.
Then we will prove that the \oneoneiahype  has a runtime that is at most a linear factor larger than that obtained for any RLS algorithm
using the artificial fitness levels method.
{\color{black}If no improvement is found in the first step}, then the operator will perform at most $cn$ useless fitness function evaluations before one hypermutation process is concluded.
{\color{black}We formalise this result in Theorem \ref{thm:afl} for two cases: when FCM only accepts strict improvements as constructive solutions (we formally call such algorithm \oneoneiahypgreater) and  for the case when FCM also accepts points of equal fitness as constructive solutions (we name such algorithm \oneoneiahypequal).}

We prove in Theorem \ref{thm:afl} that the \oneoneiahypequal cannot be too slow 
compared to the standard RLS$_1$ (i.e., flipping one bit per iteration).
We show that the presented results are tight for some standard benchmark functions 
by proving that the \oneoneiahype has expected runtimes of $\Theta(n^2\log n)$ for \textsc{OneMax} and $\Theta(n^3)$ for {\color{black}\textsc{LeadingOnes}},
respectively, versus the expected $\Theta(n\log n)$ and $\Theta(n^2)$ fitness function evaluations required by \textsc{RLS}$_1$~\cite{OlivetoBookChapter}. 
Nevertheless, we conclude the section by showing for the standard benchmark functions \textsc{Jump$_k$} and \textsc{Cliff$_d$} that the \oneoneiahype can be particularly efficient on functions 
with local optima that are generally difficult to escape from.

We start by highlighting the limitations of static hypermutation 
when FCM is not used. Since $M=cn$ distinct bits have to be flipped at once, the 
outcome of the hypermutation operator is characterised by a uniform 
distribution over the set of all solutions which have Hamming distance $M$ to the 
parent. Since $M$ is linear in $n$, the size of this set of points is 
exponentially large and thus the probability of a particular 
outcome is exponentially small. In the following theorem, we formalise 
this limitation.
{\color{black}
\begin{algorithm}[t]
\begin{algorithmic}[1]
\STATE{$t=0$,}
\STATE{initialise a solution uniformly at random and assign it to $P^{(0)}$.}
\WHILE {an optimum is not found}
\STATE{Cloning~($P^{(t)}, 1$),}
\STATE{Static hypermutation~($P^{(clo)}, c)$,}
\IF{$f(P^{(hyp)})\geq f(P^{(t)})$} 
\STATE{$P^{(t+1)}=P^{(hyp)}$.}
\ELSE
\STATE{$P^{(t+1)}=P^{(t)}$.}
\ENDIF
\STATE{$t=t+1$.}
\ENDWHILE
  \end{algorithmic}
  \caption{(1+1)~IA$^{hyp}$}
 \label{IAhype}
\end{algorithm}
}
\begin{theorem} \label{woFCM}
For any function with a polynomial number of optima, the \oneoneiahype 
without FCM needs expected exponential time to find any of the optima. 
\end{theorem}
\begin{proof}
We will first consider the probability that the initial solution is optimal and 
then the probability of finding an optimal solution in a single step given that 
the current solution is suboptimal. Note that if $c=1$, sampling the 
complementary bit string of an optimal solution allows the hypermutation to 
flip all bits in the next iteration and find an optimal solution. However, 
since each optimal solution has a single complementary bit string, the 
total number of solutions which are either optimal or complementary to an 
optimal solution is also polynomial in $n$. Since the initial solution is 
sampled uniformly at random among $2^n$ possible bit strings, the probability 
that one of these polynomially many  solutions is sampled is $ 
 poly(n)/2^{n}=2^{-\Omega(n)}$.  

We now analyse the expected time of the last step before an optimal solution is 
found given that the current solution is neither an optimal solution nor the
complementary bit string of an optimal solution. When $c=1$ the probability of 
finding the optima is zero since the hyper mutation deterministically 
samples the complementary bit strings of current b-cells.
For $c<1$, we optimistically assume that all 
the optima are at Hamming distance $M=cn$ from the current b-cell. 
Otherwise, if none of the optima are at Hamming distance $M$, the probability of reaching an optimum would be zero.
Then, given that the number of different points at Hamming distance $cn$ from any point is $\binom{n}{cn}$ and they all are reachable with equal probability,  the 
probability of finding this optimum in the last step, for any $c \neq 1$, is $p 
\leq \frac{poly(n)}{\binom{n}{cn}} \leq 
\frac{poly(n)}{e^{\Omega(n)}}={e^{-\Omega(n)}}$. 
By a simple waiting time argument, the expected time to reach any optimum is at 
least $e^{\Omega(n)}$. 
\end{proof}
The theorem explains why poor results were achieved in previous experimental 
work both on benchmark functions and real world applications such 
as the hard protein folding problem~\cite{optia,optia-transaction}.
The authors indeed state that {\it ``With this policy, however, and for the problems which are faced in this paper, 
the implemented IA did not provide good results''}~\cite{optia-transaction}. Theorem~\ref{woFCM} shows that this is the case for any optimisation function of practical interest.
In~\cite{Jansen2011} it had already been shown that inversely proportional hypermutations cannot optimise \textsc{OneMax} in less than exponential time (both in expectation and w.o.p.).
Although static hypermutations are the focus of this paper, we point out that Theorem \ref{woFCM} can easily be extended to both the inversely proportional hypermutations considered in~\cite{Jansen2011}
and to the proportional hypermutations from the literature~\cite{optia}.
From now on we will only consider hypermutations coupled with FCM. 
We start by showing that hypermutations cannot be too slow compared to local search operators.
We first state and prove the following helper lemma.

\begin{lemma} \label{lem:fcm}
The probability that the static hypermutation applied to $x \in \{0,1\}^n$ either evaluates a 
specific $y \in \{0,1\}^n$ with Hamming distance $k\leq cn$ to $x$ (i.e., 
event $E_y$), or that it stops earlier on a constructive mutation (i.e., event 
$E_c$) {\color{black}is lower bounded by $\binom{n}{k}^{-1}$ (i.e., $Pr \{ E_y \vee E_c \} \geq \binom{n}{k}^{-1}$)}. Moreover, if there are no 
constructive mutations with Hamming distance smaller than $k$, then $Pr\{E_y\}= 
\binom{n}{k}^{-1}$. 

\end{lemma}
\begin{proof}
{\color{black}Since the bits to be flipped are picked without replacement, each successive 
bit-flip increases the Hamming distance between the current solution and the 
original solution by one. 
The lower bound is based on the fact that the first $k$ bit positions to be flipped
have $\binom{n}{k}$ different and equally probable outcomes. Since the 
only event that can prevent the static hypermutation to evaluate the $k$-th 
solution in the sequence is the discovery of a constructive mutation in one of 
the first $k-1$ evaluations, $Pr \{ E_y \vee E_c \}$ {\color{black}is at least} $\binom{n}{k}^{-1}$. If no such constructive mutation exists (i.e., $Pr \{  E_c \}=0$) then, $Pr \{  E_y \}$ is exactly equal to $\binom{n}{k}^{-1}$.} 
\end{proof}

We are ready to show that static hypermutations cannot be too slow compared to  
any upper bound obtained by applying the AFL method 
on the expected runtime of the \oneonerlsk{\color{black},} which flips exactly $k$ bits to produce a 
new solution and applies non-strict elitist selection.
AFL require{\color{black}s} a partition of the search space $\mathcal{X}$ into $m$ mutually 
exclusive sets  $\bigcup_{s\in \{1,\ldots,m\}} A_{s}=\mathcal{X}$ such that $\forall 
i<j, \quad x \in A_i \wedge y \in A_j \implies f(x) < f(y)$. 
The expected runtime of a $(1+1)$ algorithm $\mathcal{A}$ with 
variation operator \text{HM}$(x):\mathcal{X} \rightarrow \mathcal{X}$  to 
solve any 
function defined on $\mathcal{X}$ can be upper bounded via AFL
by $E(T) \leq \sum_{i=1}^m \frac {1}{p_i}$, where $p_{i}=\min\limits_{x\in A_i}
\left(Pr\{\text{HM}(x) \in \bigcup\limits_{j=i+1}^{m} A_j\} 
\right)$~\cite{OlivetoBookChapter, jansenbook}. 

{\color{black}
\begin{theorem} \label{thm:afl}
Let $E\left(T^{AFL}_{A}\right)$ be any upper bound on the expected runtime of algorithm A established via the artificial fitness levels method. Then, $E\left(T_{(1+1)~\text{IA}_{>}^{hyp}}\right) \leq cn \cdot E\left(T^{AFL}_{(1+1)~RLS_k}\right)$. Additionally, for the special case of $k=1$, $E\left(T_{(1+1)~\text{IA}_{\geq}^{hyp}}\right) \leq cn \cdot E\left(T^{AFL}_{(1+1)~RLS_1}\right)$.
\end{theorem}
\begin{proof}
Let the function $c_k(x)$ for solution $x\in A_i$ return the number of 
solutions which are at Hamming distance $k$ away from $x$ and belong to set 
$A_j$ 
for some $j>i$.
The upper bound on the expected runtime of the \oneonerlsk to solve any 
function obtained by applying the AFL method is
$E\left(T^{AFL}_{(1+1)~RLS_k}\right) \leq \sum_{i=1}^m \frac {1}{p_{i}}$,
where
$p_{i}=\min\limits_{x\in A_i}
\left(c_k(x)/\binom{n}{k} \right)$.
Since the hypermutation operator wastes at most $cn$ bit mutations when it fails to improve, to prove 
the first claim it is sufficient to show that for any current solution $x\in A_i$, the 
probability that the \oneoneiahypgreater finds an improvement is at least 
$c_k(x)/\binom{n}{k}$. This follows from Lemma~\ref{lem:fcm} and the 
definition of a constructive mutation for the \oneoneiahypgreater, since for each one of
the $c_{k}(x)$ search points, the probability of either finding it or finding a 
constructive mutation is lower bounded by $\binom{n}{k}^{-1}$. 

Note that, for the \oneoneiahypequal, the probability of improving the 
current solution $x\in A_i$ can be smaller than $c_k(x)/\binom{n}{k}$ since it 
is also necessary that the first $k-1$ sampled solutions  have strictly worse 
fitness than $x$ to allow the hypermutation operator to flip at least $k$ bits 
before stopping. However, for the special case of $k=1$, the hypermutation 
cannot be prematurely stopped. The probability that static hypermutation 
produces a particular Hamming neighbour of the input solution in the first 
mutation step is $1/n$, which is equal to the probability that the 
$\text{RLS}_1$ produces the same solution. Considering the fact that static 
hypermutation wastes at most $c n$ fitness evaluations in every failure to 
obtain an improvement in the first step, the second claim follows. 
%
\end{proof}
 }


In the following we show that the upper bounds of the previous theorem are 
tight for well-known benchmark functions.
\begin{theorem} \label{onemax}
The expected runtime of the \oneoneiahypgreater and of the \oneoneiahypequal to optimise 
\textsc{OneMax}$(x):=\sum_{i=1}^n x_i$
is $\Theta(n^2 log\,n)$.
\end{theorem}

\begin{proof}
The upper bounds for the $>$ and $\geq$ FCM selection versions follow from Theorem \ref{thm:afl}
since it is easy to derive an upper bound of $O(n \log n)$ for RLS using AFL~\cite{OlivetoBookChapter}. 
For the lower bound of the $>$ FCM selection version, we follow the analysis 
in Theorem 3 of~\cite{Jansen2011} for inversely proportional 
hypermutation (IPH) with FCM to optimise \textsc{ZeroMin}. 
The proof there relies on IPH wasting $cn$ function evaluations every time it fails to find an improvement.
This is obviously also true for static hypermutations (albeit for a different constant $c$), hence the proof also applies to our algorithm.
The lower bound also holds for the $\geq$ FCM selection algorithm as this algorithm cannot be faster 
on \textsc{OneMax} by accepting solutions of equal fitness (i.e., the probability 
of finding an improved solution does not increase). 
\end{proof}
We now turn to the \textsc{LeadingOnes} benchmark function, which simply returns the number of consecutive $1$-bits before the first 0-bit.

\begin{theorem}
The expected runtime of the \oneoneiahypequal on 
\textsc{LeadingOnes}$:= \sum_{i=1}^{n}\prod_{j=1}^{i}x_i$
is $\Theta(n^3)$.
\end{theorem}

\begin{proof}
The upper bound is implied by Theorem \ref{thm:afl} because AFL gives 
an $O(n^2)$ runtime of RLS for \textsc{LeadingOnes}~\cite{OlivetoBookChapter}.
Let {\color{black}$E(f_i)$} be the expected number of fitness function evaluations until an improvement is 
made, considering that the initial solution has $i$ leading 1-bits. The 
initial 
solutions consist of $i$ leading 1-bits, followed by a 0-bit and $n-i-1$ 
bits which each can be either one or zero with equal probability. Let events $E_1$, 
$E_2$ and $E_3$ be that the first mutation step flips one of the 
leading ones (with probability $i/n$), the first 0-bit (with probability $1/n$) 
or any of the remaining bits (with probability $(n-i-1)/n$), respectively. If $E_1$ 
occurs, then the following mutation steps cannot reach any solution 
with fitness value $i$ or higher and all $cn$ mutation steps are 
executed. Since no improvements have been achieved, the remaining expected 
number of evaluations will be the same as the initial expectation (i.e., 
{\color{black}$E(f_i|E_1)=cn+E(f_i)$}). If $E_2$ occurs, then a new solution with higher 
fitness value is acquired and the mutation process stops (i.e.,
{\color{black}$E(f_i|E_2)=1$}). However if $E_3$ occurs, since the number of leading 1-bits in 
the new solution is $i$, the hypermutation operator stops without any improvement (i.e., {\color{black}$E(f_i|E_3)=1+E(f_i)$}). According to the law of total expectation:
{\color{black}$ E(f_i) = \frac{i}{n} \left(cn + E(f_i)\right) + \frac{1}{n} \cdot 1 
+\frac{n-i-1}{n} \left( 1+ E(f_i)\right) $}. When, this equation is solved for 
{\color{black}$E(f_i)$}, we obtain, {\color{black}$E(f_i)= i cn + n - i$}.
%
%
Since the expected number of consecutive 1-bits that follow the leftmost 0-bit is less than two \cite{droste}, 
the probability of not skipping a level $i$ is $\Omega(1)$. The initial 
solution on the other hand will have more than $n/2$ leading ones with 
probability at most $2^{-n/2}$. Thus, we obtain a lower 
bound on the expectation, 
{\color{black}$(1-2^{-n/2})\sum\limits_{i=n/2}^{n}(f_i)= 
\Omega(1)\sum\limits_{i=n/2}^{n} (i cn + n - i)=\Omega(n^3)$} by summing over 
fitness levels starting from level $i=n/2$.
%
\end{proof}
We now focus on establishing that hypermutations may produce considerable speed-ups if local optima need to be overcome. 
The \textsc{Jump}$_{k}$ function, introduced in~\cite{droste}, consists of a 
\textsc{OneMax} slope with a gap of length $k$ bits that needs to be overcome 
for the optimum to be found.
The function is formally defined as:

 \begin{figure}[t!]
 \centering
  \includegraphics[width=0.5\textwidth]{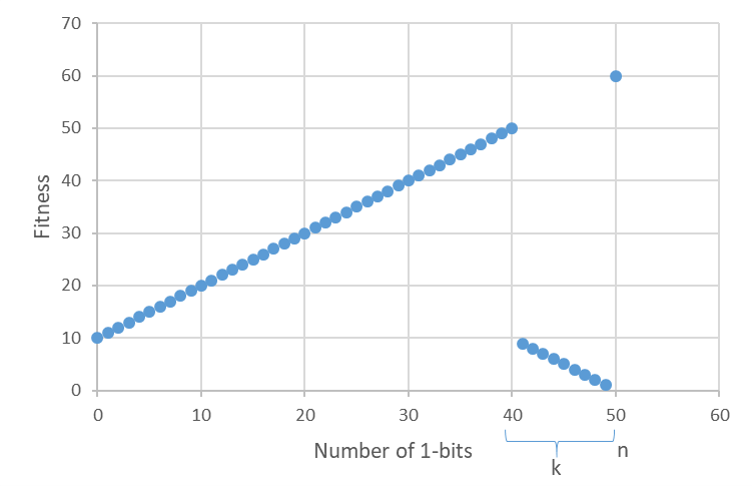}
 \caption{{\color{black} The $\textsc{Jump}_{k}$ function for $k=10$ and $n=50$.}}
 \label{fig:jump}
 \end{figure}

\begin{align*}
\textsc{Jump}_{k}(x):= \begin{cases}
		k+\sum_{i=1}^n x_i & \text{if}\; \sum_{i=1}^n x_i \leq n-k\\
        &\text{or}\; \sum_{i=1}^n x_i=n, \\
        n-\sum_{i=1}^n x_i & \text{otherwise},
 \end{cases}
\end{align*}
for $n>1$ and $k \in \{1...n\}$. {\color{black}Figure \ref{fig:jump} illustrates this function}.

Mutation-based EAs require $\Theta(n^k)$ expected function evaluations to optimise \textsc{Jump}$_{k}$ and recently a faster upper 
bound by a linear factor has been proved for standard crossover-based steady-state GAs~\cite{OlivetoTEVC2017}.
Hence, EAs require increasing runtimes as the length of the gap increases, from superpolynomial to exponential as soon as $k=\omega(1)$.
The following theorem shows that hypermutations allow speed-ups by an exponential factor of $(e/k)^k$, when the jump is hard to perform. A similar result has been shown for the recently introduced fast-GA~\cite{DoerrGecco2017}.

\begin{theorem}
Let $cn>k$. Then the expected runtime of the \oneoneiahype to optimise \textsc{Jump$_k$} is at most $O(\frac{n^{k+1} \cdot e^k}{k^k})$. 
\end{theorem}

\begin{proof}
The \oneoneiahype reaches the fitness level $n-k$ (i.e., {\color{black}local 
optima}) 
in $O(n^2 \log\,n)$ steps according to Theorem \ref{onemax}. 
All local optima have Hamming distance $k$ to the optimum and the 
probability that static hypermutation finds the optimum is lower bounded in 
Lemma~\ref{lem:fcm}  by $\binom {n}{k}^{-1}$. Hence, the total expected time to 
find the optimum is at most $O(n^2 log\,n) +cn \cdot \binom {n}{k}= 
O\left(\frac{n^{k+1} \cdot e^k}{k^k}\right)$.
\end{proof}
Obviously, hypermutations can jump over large fitness valleys also on functions with other characteristics.
For instance the \textsc{Cliff}$_{d}$ function was originally introduced to show when non-elitist EAs may outperform elitist ones~\cite{JaegerskuepperStorch}. {\color{black}This function is formally defined as follows:}
{\color{black}
\begin{displaymath}
\textsc{Cliff}_{d}(x)=\begin{cases}
		\sum_{i=1}^n x_i & \text{if}\; \sum_{i=1}^n x_i \leq n-d, \\
        \sum_{i=1}^n x_i -d+ 1/2 & \text{otherwise.}
 \end{cases}
\end{displaymath}
}

 \begin{figure}[t!]
 \centering
  \includegraphics[width=0.5\textwidth]{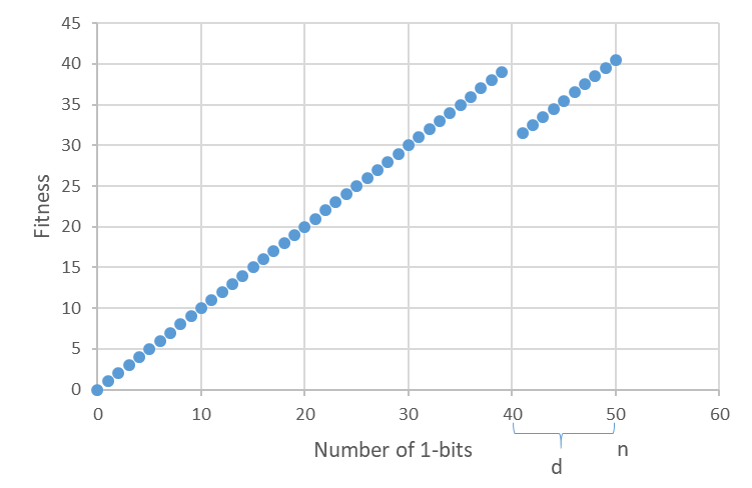}
 \caption{{\color{black} The $\textsc{Cliff}_{d}$ function for $d=10$ and $n=50$.}}
 \label{fig:cliff}
 \end{figure}
 
{\color{black}Figure \ref{fig:cliff} shows an illustration of \textsc{Cliff}$_{d}$}. Similarly to \textsc{Jump$_k$}, this function has a \textsc{OneMax} slope with a gap of length $d$ bits that needs to be overcome for the optimum to be found. 
Differently to \textsc{Jump$_k$} though, the local optimum is followed by another \textsc{OneMax} slope leading to the optimum. Hence, algorithms
that accept a move jumping to the bottom of the cliff, can then optimise the following slope and reach the optimum.
While elitist mutation-based EAs obviously have a runtime of $\Theta(n^d)$ (i.e., they do not accept the jump to the bottom of the cliff),
the following corollary shows how hypermutations lead to speed-ups that increase exponentially with the distance $d$ between the cliff and the optimum.

\begin{corollary}
Let $cn>d$. Then the expected runtime of the \oneoneiahype to optimise \textsc{Cliff$_d$} is $O\left(\frac{n^{d+1} \cdot e^d}{d^d}\right)$. 
\end{corollary}
The analysis can also be extended to show an equivalent speed-up compared to the \oneoneea for crossing the fitness valleys of arbitrary length and depth recently introduced in~\cite{Oliveto2017}. In the next section we will prove that the ageing operator can lead to surprising speed-ups for \textsc{Cliff$_d$} and other functions with similar characteristics.

\section{Ageing} \label{sec:ageing}
It is well-understood that the ageing operator can allow algorithms to escape from local optima.
This effect was shown on the \textsc{Balance} function from dynamic optimisation for an RLS algorithm embedding a hybrid ageing operator~\cite{ageing}. 
However, for that specific function, an SBM operator would fail to escape, due to the large basin of attraction of the local optima.
In this section we highlight the capabilities of ageing in a more general setting (i.e., the standard \textsc{Cliff$_d$} benchmark function) and show that ageing may also be efficient  when coupled 
with SBM.

Ageing allows to escape from a local optimum if one not locally optimal b-cell is created and survives while all the other b-cells die.
For this to happen, it is necessary that all the b-cells are old and have similar age. 
This is achieved on a 
local optimum by creating copies of the locally optimal b-cell (the 
b-cells will inherit the age of their parent).
Hence, the ability of a mutation operator to create copies enhances this 
particular capability of the ageing operator.
To this end we
first consider a modified RLS algorithm 
that with some constant probability $p$ does not flip any bit and 
implements the ageing operator presented in 
Algorithm~\ref{alg:h-ageing}. We call this algorithm 
 \muonerlsageing and present its pseudo-code in Algorithm \ref{ageing}. 
Apart from making ageing more effective, this slight modification to the standard RLS considerably simplifies the proofs of the statements we wish to make. 
In Section \ref{sec:diversity} we will generalise the result to an RLS algorithm that does not allow genotype duplicates as in the standard Opt-IA.

\begin{algorithm}[t]
\begin{algorithmic}[1]
\STATE{$t=0$,} 
\STATE{{\color{black}initialise $P^{(t)}=\{x_1,\cdots,x_\mu\}$, a population of $\mu$ individuals uniformly at random and set $x_i^{age}=0$ for $i=\{1,\cdots,\mu\}$}.}
\WHILE {the optimum is not found}
\STATE{ select $x\in P^{(t)}$ uniformly at random,}
\STATE{ with probability $1/2<1-p\leq 1$ create $y$ by flipping one bit of $x$, otherwise $y=x$,}
\STATE{Hybrid ageing~$\left(P^{(t)},\{y\}, \tau, \mu \right)$,}
\STATE{Selection $\left(P^{(t)},\{y\}, \mu, 0 \right)$,} 
\STATE{$t=t+1$.} 
\ENDWHILE
  \end{algorithmic}
  \caption{\muonerlsageing}
 \label{ageing}
\end{algorithm}

The \textsc{Cliff$_d$} benchmark function is generally used to highlight circumstances when non-elitist EAs outperform elitist ones.
Algorithms that accept inferior solutions can be efficient for the function by jumping to the bottom of the cliff and then optimising the \textsc{OneMax} slope.
This effect was originally shown for the $(1,\lambda)~\text{EA}$ that can optimise the function in approximately $O(n^{25})$ fitness function evaluations if the population size $\lambda$
is {\color{black}neither} too large nor too small~\cite{JaegerskuepperStorch}. This makes the difference between polynomial and exponential expected runtimes compared to elitist EAs (i.e., $\Theta(n^d)$) 
if the cliff is located far away from the optimum. A smaller, but still exponential, speed-up was recently shown for the {\color{black}population-genetics-inspired} SSWM {\color{black}(Strong Selection Weak Mutation)} algorithm with runtime {\color{black}of}
at most $n^d/e^{\Omega(d)}$~\cite{Jorge2015}. 
 The following 
theorem proves a surprising result for
the considered \muonerlsageing for \textsc{Cliff$_d$}. 
Not only is the algorithm very fast, but our upper bound becomes lowest (i.e., $O(n\log n)$) when the function is most difficult (i.e., when the cliff is located at distance 
$d=\Theta(n)$ from the optimum). In this setting the algorithm is asymptotically as fast as any evolutionary algorithm using {\color{black}SBM} can be on any function 
with unique optimum~\cite{Sudholt2012}. A similar result has recently been shown also for a hyperheuristic which switches between elitist and non-elitist selection operators \cite{LissovoiOlivetoWarwicker2019}.

\begin{theorem} \label{RLSp}
For $\mu=O(\log n)$, $p<1$ a constant and $\tau=\Theta( n \log n)$, 
the \muonerlsageing
optimises \textsc{Cliff$_d$} in expected time
$O\left(\frac{\mu^2 n^3 \log n}{d^2}\right)$ if $d< n/4 - \epsilon n$ for any 
constant 
$0<\epsilon<1/4$.
\end{theorem}

\begin{proof}
{\color{black}We follow the proof of Theorem 10 in~\cite{ageing} of the $(\mu+1)$~RLS for the \textsc{Balance} function and adapt the arguments therein to the \textsc{OneMax} landscape and to the RLS operator we use. 
Given that there are $i<\mu$ individuals with $j$ 1-bits in the population, 
the probability of creating a new individual with $j$ 1-bits is at least $(i/\mu) p$
because the RLS operator creates a copy with a constant probability $p$. 
Hence we follow the proof in~\cite{ageing} to show that in $O(\mu n + n \log{n})$ expected steps
 the population climbs up the \textsc{OneMax} slope  (i.e., samples a solution 
with $n-d$ 1-bits) and subsequently the whole population will be taken over by 
the local optimum. Given that there are already $k$ copies of the best 
individual in the population, the probability that one is selected for mutation 
is $k/\mu$ and the probability that a new copy is added to the population is 
$kp/\mu$. Thus in at most $\sum_{k=1}^{\mu}\mu/kp=O(\mu \log{\mu})$ generations 
in expectation after the first local optima is sampled, the population is taken 
over by the local optima.
Now we can apply Lemma 5 in~\cite{ageing} to show that in expected $O(\mu^3)$ steps the whole population will have the same age. 
As a result, after another at most $\tau=\Theta(n \log{n})$ generations the whole population will reach age $\tau$ simultaneously because no improvements
may occur unless the optimum is found.
Overall, the total expected time until the population consists only of local optima with age $\tau$ is at most 
$O(\mu n + n \log{n} + \mu \log{\mu}+ \mu^3+\tau)=O(\mu n + n \log{n})$. 
Now we calculate the probability that in the next step one individual jumps to 
the bottom of the cliff and the rest die in the same generation.
The first event happens with probability $(1-p)(d/n)=\Omega\left(d/n \right)$ (i.e., an offspring solution with $n-d+1$ 1-bits is created by flipping one of the $d$ 0-bits in the parent solution). The probability that the rest of the population dies is $1/\mu \cdot (1-1/\mu)^{\mu}$.
We now require that the survivor creates an offspring with higher fitness (i.e., with $age=0$) by flipping one of its 0-bits (i.e., it climbs one step up the second slope). This event happens with probability
at least $(1-p)(d-1)/(\mu n)=\Omega\left(\frac{d}{\mu n} \right)$ 
and in the same generation with probability $(1-1/\mu)=\Omega(1)$ the parent of age $\tau+1$ dies due to ageing.
Finally, the new solution (i.e., the {\color{black}{\it safe}} individual) takes 
over the population in $O(\mu \log{\mu})$  expected generations by standard 
arguments. It takes at 
least $\Omega(n)$ generations before any of the new random individuals 
has $\epsilon n$ more 0-bits than it had after initialisation since 
FCM terminates after each improvement. Using Markov's inequality, we can 
show that the probability that the takeover happens after more than 
$\Omega(n)$ steps is at most  $O((\mu \log{\mu})/n)$. 

Hence, the overall probability of this series of consecutive events is 
$\Omega\left(\frac{d}{n} \right)\cdot\frac{1}{\mu}\left(1-\frac{1}{\mu}\right)^{\mu}  
 \cdot 
\Omega\left(\frac{d}{\mu n} \right)\cdot (1-\frac{1}{\mu})\cdot \left(1-O\left(\frac{\mu 
\log{\mu}}{n}\right)\right) = \Omega\left( \frac{d^2}{n^2 \mu^2 }\right)$
%
and the expected number of trials (i.e., climb-ups and restarts) until we get a 
survivor which is safe at the bottom of the cliff is
$O\left(\frac{n^2 \mu^2 }{d^2}\right)$.
 Every time the set of events fails to happen, we wait for another 
$O(\mu n + n \log n)$ fitness evaluations until the population reaches  a configuration where all individuals are locally optimal and have age $\tau$. 
 Once a safe individual has taken over the population, the expected time to find the global 
optimum will be at most $O(\mu n + n \log{n})$.
Overall, the total expected time to optimise  \textsc{Cliff$_d$} conditional on the best individual never dying when climbing up the slopes is $E(T_{total}) 
 \leq O(\mu n + n \log n)\cdot O\left(\frac{n^2 \mu^2 
}{d^2}\right)+O\left(\mu n + n \log n\right)=O \left(\frac{\mu^2 n^3 \log 
n}{d^2}\right)$. 
Finally, we consider the probability that the best individual in the 
population never dies when climbing up the slopes due to ageing.
After any higher fitness level is discovered, it takes $O(\mu \log{\mu})$ 
generations in expectation and at most $n^{1/2}$ 
generations with overwhelming probability (w.o.p.\footnote{{\color{black}In the rest of the paper we consider events to occur  ``with overwhelming probability'' (w.o.p.) meaning that they occur with probability at least $1 - 2^{-\Omega(n)}$}.}) until the whole population takes over the level.
For the first $n-\log{n}$ levels, the 
probability of improving a solution is at least 
$\Omega(\log{n}/n)$ and the probability that this improvement does not happen 
in $\tau-n^{1/2}=\Omega(n \log{n})$ generations is at most  $(1- 
\Omega(\log{n}/n))^{\Omega(n 
\log{n})}=e^{-\Omega(\log^2{n})}=n^{-\Omega(\log{n})}$.  For the remaining 
fitness levels, the probability of reaching age $\tau$ before improving is 
similarly $(1-\Omega(1/n))^{\Omega(n \log{n})}= 
e^{-\Omega(\log{n})}=n^{-\Omega(1)}$. By the union bound over all levels, the probability that the best solution is never lost due 
to ageing is at least $1-o(1)$. 
We pessimistically assume that the whole optimisation process has to restart 
if the best individual reaches age $\tau$. However, since at most 
$1/(1-o(1))=O(1)$ 
restarts are necessary in expectation, our bound on the expected 
runtime holds.
}
\end{proof}

We conclude the section by considering the \muoneeaageing which differs from the\\ {\color{black}\muonerlsageing} by using {\color{black}SBM} with mutation rate $1/n$ instead of flipping exactly one bit. SBM allows copying individuals 
but, since it is a global operator, it can jump back 
to the local optima from anywhere in the search space with non-zero 
probability.
Nevertheless, the following theorem shows that, for not too large populations, the algorithm is still very 
efficient when the cliff is at linear distance from the optimum. Its 
proof follows similar arguments to those of 
Theorem~\ref{RLSp}. The main difference in the analysis is that it has to be shown that
 once the solutions have jumped to the bottom of the cliff, they have a good probability of reaching the optimum before jumping back to the top of the cliff.

\begin{theorem}\label{thm:eacliff}
The \muoneeaageing optimises \textsc{Cliff$_d$} in expected 
$O(n^{1+\epsilon} )$ time if $d=(1/4)(1-c)n$ for some constant $0 
< c < 1$, $\tau=\Theta(n \log{n})$ and 
$\mu=\Theta(1)$, where $\epsilon$ is an arbitrarily  small positive constant.
\end{theorem}

\begin{proof}
The probability that the SBM operator increases the number of 1-bits in a 
parent solution with  $\Omega(n)$  0-bits is at least 
$\Omega(n)/(ne)=\Omega(1)$. Following the same arguments as in the proof of 
Theorem~\ref{RLSp} while considering $n/4 > d =\Theta(n)$ and $\mu=O(1)$ we can 
show that with constant probability and in expected $O(n \log{n})$ time the 
algorithm reaches a configuration where there is a single individual at the 
bottom of the cliff with $n-d+2$ 1-bits and age zero while the rest of the 
population consists of solutions which have been randomly sampled in the 
previous iteration. 

{\color{black}
We will now show that the newly generated individuals have worse fitness 
than the solution at the bottom of the cliff when they are initialised. Since 
$d=(1/4)(1-c)n$, the fitness value of the solutions with more than $n-d$ 
1-bits is at least $n-(1/4)(1-c)n-(1/4)(1-c)n=(n/2)(1+c)$. Due to Chernoff 
bounds, the newly created individuals have less than $(n/2)(1+(c/2))<(n/2)(1+c)$ 
1-bits with overwhelming probability.

Next we prove that for any constant $\epsilon$, there exist some 
positive constant $c^{*}$, such that the best solution at the bottom of the 
cliff (the leading solution) will be improved consecutively for $c^{*} \log {n}$ 
iterations with probability at least $n^{-\epsilon}$.  The leading solution 
with $i$ 0-bits is selected for mutation and SBM flips a single 0-bit with 
probability $p_i=(1/\mu)\cdot (i/n) (1-n^{-1})^{n-1}$. With probability 
$\prod_{i=n-d+2}^{n-d+1+c^{*} \log {n}}p_i$, $c^{*} \log {n}$ consecutive 
improvements occur. We can bound this probability from below by the final 
improvement probability $p_{f}$ raised to the power of $c^{*} \log {n}$ since 
the improvement probability is inversely proportional to the number of 0-bits. 
Considering that $p_{f}\geq(n-d+1+c^{*} \log {n})/(n\mu e)=\Omega(1)$, we can 
set $c^{*}:= -\epsilon \log_{p_f}{2}=\Omega(1)$, which yields $p_{f}^{c^{*} 
\log {n}}= n^{-\epsilon}$. Here, we note that immediately after $c^{*} \log 
{n}$ consecutive improvements, all the individuals in the population have at 
least 
$n-d+1+c^{*} \log {n}- \mu $ 1-bits. More precisely there will be one and only 
one individual in the population with $j$ bits for all $j$ in $[n-d+1+c^{*} 
\log {n}- \mu, n-d+2+c^{*} \log {n}]$. Since $\mu$ is constant and SBM flips 
at least $k$ bits with probability $n^{-k}\binom{n}{k}$, the probability that a 
single operation of SBM decreases the number of 1-bits in any individual below 
$n-d+1$ is in the order of $O(1/(\log{n})!)$. Since this probability is not 
polynomially bounded, with probability at least $1- O(1/n)$ it will not happen 
in any polynomial number of iterations. After the consecutive improvements occur, 
it takes $O(n \log {n})$ until the second slope is climbed and the optimum is 
found. 

The asymptotic bound on the runtime is obtained by considering that 
algorithm will reach the local optimum $n^{\epsilon}$ times in expectation 
before the consecutive improvements are observed. Since the time to restart 
after reaching the local optimum is in the order of $O(n \log {n})$, the 
expected time is $O(n^{1+\epsilon} \log {n})$. Since $\epsilon$ is an 
arbitrarily small constant, the order $O(n^{1+\epsilon} \log {n})$ is 
equivalent to the order $O(n^{1+\epsilon})$.
}
\end{proof}

\section{Opt-IA} \label{sec:optia}

After having analysed the operators separately, in this section we consider the complete 
Opt-IA. The considered Opt-IA, shown in Algorithm \ref{modified-optia}, uses 
static hypermutation coupled with FCM as variation operator, hybrid ageing 
and an standard $(\mu+\lambda)$ selection {\color{black}which allows genotype duplicates}. Also, a mutation is considered 
constructive if it results 
in creating an equally fit solution or a better one. 

In this section we first show that Opt-IA is efficient for all the functions considered previously in the paper. Then, in Subsection \ref{subsec:optiaefficient} we present a problem where the use of the whole Opt-IA is crucial. In Subsection \ref{subsec:optiafail} we show limitations of Opt-IA by presenting a class of functions where standard EAs are efficient while Opt-IA is not w.o.p. We conclude the section with Subsection \ref{subsec:trap} where we present an analysis for trap functions which disproves previous claims in the literature about Opt-IA's behaviour on this class of problems. 

The following theorem proves that Opt-IA 
can optimise any benchmark function considered previously in this 
paper. The theorem uses that the ageing parameter $\tau$ is set large enough such that no individuals die with high probability
before the optima are found.
\begin{theorem} \label{th:optiallfunctiongendive}
Let $\tau$ be large enough.
Then the following upper bounds on the expected runtime of Opt-IA hold: 
$
\\{\color{black}E(T_{\textsc{OneMax}})} ~=~O\left(\mu \cdot dup \cdot  n^2 \log{n}\right)\; \text{for} \; \tau=\Omega(n^2),\;
\\{\color{black}E(T_{\textsc{LeadingOnes}})}~=~ O\left(\mu \cdot dup \cdot  n^3\right)\; \text{for} \; \tau=\Omega(n^2),\;
\\{\color{black}E(T_{\textsc{Jump}_k})} ~=~ O\left(\mu \cdot dup \cdot \frac{n^{k+1} \cdot e^k}{k^k}\right)\; \text{for} \; \tau=\Omega(n^{k+1}),\;
\\ \text{and\;\;} 
{\color{black}E(T_{\textsc{Cliff}_d})} ~=~ O\left(\mu\cdot dup \cdot \frac{n^{d+1} \cdot e^d}{d^d}\right)\; \text{for} \; \tau=\Omega(n^{d+1}{\color{black})}\; 
$.
\label{optiaallfunc}
\end{theorem}

\begin{proof}
 The claims use that if  $\tau$ is large enough (i.e., $\Omega(n^2)$ for \textsc{OneMax} and \textsc{LeadingOnes}, $\Omega(n^{k+1})$ for $\textsc{Jump}_k$ and $\Omega(n^{d+1})$ for $\textsc{Cliff}_d${\color{black})}, then with probability $1-o(1)$ the current best solution will never reach age $\tau$ and die due to ageing before an improvement 
is found. For the Opt-IA to lose the current best solution due to ageing, it is necessary that the best solution is not improved for $\tau$ generations consecutively. If the improvement probability for any non-optimal solution is at least $p_{min}$ and if the age $\tau$ is set to be $p_{min}^{-1} n$, then the probability that a solution will reach age $\tau$ before creating an offspring with higher fitness is at most $(1-p_{min})^{\tau}\leq e^{-n}$. By the union bound it is also exponentially unlikely that this occurs in a polynomial number of fitness levels that need to be traversed before reaching the optimum.  Since the suggested $\tau$ for each function is larger than the corresponding $p_{min}^{-1} n$, the upper bounds of the \oneoneiahype (which does not implement ageing)
for \textsc{OneMax}, \textsc{LeadingOnes}, \textsc{Jump$_k$} and \textsc{Cliff$_d$} are 
valid for Opt-IA when multiplied by  $\mu \cdot dup$ to take into account the 
population and clones. 
%
% ...
%
%
%
\end{proof}

Theorem \ref{thm:lbonemax} shows that the presented upper bound for \textsc{OneMax} is tight for sub-logarithmic population and clone sizes (i.e., $\mu=o(\log n)$ and $dup=o(\log n)$). {\color{black}Before stating this theorem, we state the following helper theorem which was already used in \cite{Jansen2011} to prove lower bounds on the expected runtime of inversely proportional hypermutations.

\begin{theorem}[Ballot Theorem~\cite{feller1968}] \label{thm:ballot} 
``Suppose that, in a ballot, candidate P scores $p$ votes and candidate Q scores $q$ votes, where $p>q$. The probability that throughout the counting there are always more votes for P than for Q equals $(p-q)/(p+q)$''. 
\end{theorem}
This theorem allows us to derive an upper bound on the probability that at some point of the hypermutation we have picked more 0-bits than 1-bits, which implies an improvement by hypermutation for \textsc{OneMax}.

}
\begin{theorem} \label{thm:lbonemax}
Opt-IA needs at least $cn^2 (\frac{\log (n/3)}{2}-\frac{\mu \cdot dup}{3})$ 
expected fitness function evaluations for any mutation potential $c$ to optimise 
\textsc{OneMax}.
\end{theorem}

\begin{proof}
By Chernoff bounds, the initial individuals have at least $i=n/3$ 
0-bits with overwhelming probability. To calculate the probability of an improvement (i.e., flipping equal or more 0-bits than 1-bits during one mutation operation), we use the Ballot theorem ({\color{black}i.e., Theorem \ref{thm:ballot}}) in a similar way to~\cite{Jansen2011}. Considering the number of 
0-bits as $i=q$ and the number of 1-bits as $n-i=p$, the probability of 
an improvement is at most $1-(p-q)/(p+q)=1-(n-2i)/n=2i/n$ according to the Ballot theorem\footnote{Like in~\cite{Jansen2011} we consider that using $cn<n$ mutations, the probability can only be lower than the result stated in the Ballot theorem.}. 
Hence, the probability that at least one out of $dup \cdot \mu$  individuals succeeds is $ 
P \leq dup \cdot \mu 2i/n$ by the {\color{black}union bound}. We optimistically 
assume that the rest of the individuals also improve their fitness after such event happens. Recall that the mutation operator wastes $cn$ fitness function evaluations every time it does not improve the fitness. Therefore, the expected time to see an improvement is
$E(T_{improve}) \geq \left(\frac{n}{dup \cdot 2 \mu i}-1 \right) \cdot \mu cn \cdot dup$. 
Since the mutation operator stops at the first constructive mutation (i.e., when the number of 1-bits is increased by one), it is necessary to improve at least $n/3$ times. So the total expected time to optimise \textsc{OneMax} is {\color{black}$E(T_{total}) \geq \sum_{i=1}^{n/3} E(T_{improve})= cn^2 (\frac{\log n/3}{2}-\frac{\mu \cdot dup}{3})$. }

{\color{black}If individuals were to be removed because of ageing, then the new randomly generated individuals that replace them will have to improve at least $n/3$ times all over again  w.o.p. Hence, the runtime may only increase in such an event.}
\end{proof}

\subsection{Opt-IA Can Be More Efficient} \label{subsec:optiaefficient}
In this section, we present the function \textsc{HiddenPath} to illustrate a 
problem where the use of static hypermutation and ageing together is 
crucial. When either of these two characteristic 
operators of Opt-IA is not used, we will prove that the expected runtime is at 
least superpolynomial. {\color{black}$\textsc{HiddenPath}\colon\{0,1\}^{n}\rightarrow 
\mathbb{R}$} can be described by a series of modifications to the well-know 
\textsc{ZeroMax} function. The distinguishing solutions are those with 
five 0-bits and those with $n-1$ 0-bits, 
along with $\log{n}-3$ solutions of the form $1^{n-k}0^{k}$ for $5\leq k \leq 
\log{n} +1$ (called \textsc{Sp} points). The solutions with exactly $n-1$ 0-bits constitute the local 
optima of \textsc{HiddenPath} (called \textsc{LocalOpt}), and the solutions with exactly five 0-bits form a gradient with fitness increasing with more 0-bits in the rightmost five bit positions. {\color{black}Given that $|x|_0$ and $|x|_1$ respectively denote the number of 0-bits and 1-bits in a bit string $x$, \textsc{HiddenPath} is formally defined as in Definition \ref{def:hiddenpath}. \textsc{HiddenPath} is illustrated in Figure~\ref{hp} where \textsc{Opt} shows the global optimum.

\begin{definition} \label{def:hiddenpath}
Given the definitions of \textsc{Sp} and \textsc{LocalOpt} as above, for any positive constant $\epsilon<1$, the \textsc{HiddenPath} function is defined for all $x \in\{0,1\}^n$ by

\begin{displaymath}
\textsc{HiddenPath}(x)= 
\begin{cases}
		n-\epsilon + \frac{\sum_{i=n-4}^n (1-x_i)}{n}  & \text{if}\; |x|_0=5 \text{ and } x \neq 1^{n-5}0^5,\\
        0 & \text{if}\; |x|_0<5 \text{ or } |x|_0=n,\\
       n-\epsilon+\epsilon k/\log n & \text{if } x \in \textsc{Sp}:=\{x \mid x=1^{n-k}0^k\; and 
\; 5 \leq k \leq \log{n}+1\},  \\
       n & \text{if}\; x \in \textsc{LocalOpt}:=\{x \mid |x|_0=n-1\},\\
      |x|_0 & \text{otherwise}.
 \end{cases}
\end{displaymath}

\end{definition}
}

Since the all 0-bits string returns 
fitness value zero, there is a drift towards solutions with $n-1$ 0-bits
while the global 
optimum (\textsc{Opt}) is the $1^{n-\log{n}-1}0^{\log{n}+1}$ bit string. The solutions 
with exactly five 0-bits work as a net that stops any static hypermutation 
that has an input solution with less than five 0-bits. {\color{black}The path to the global optimum} consists of  $\log{n}-3$ Hamming neighbours and the first 
solution on this path has five 0-bits.  

 \begin{figure}[t!]
 \centering
  \includegraphics[width=0.6\textwidth]{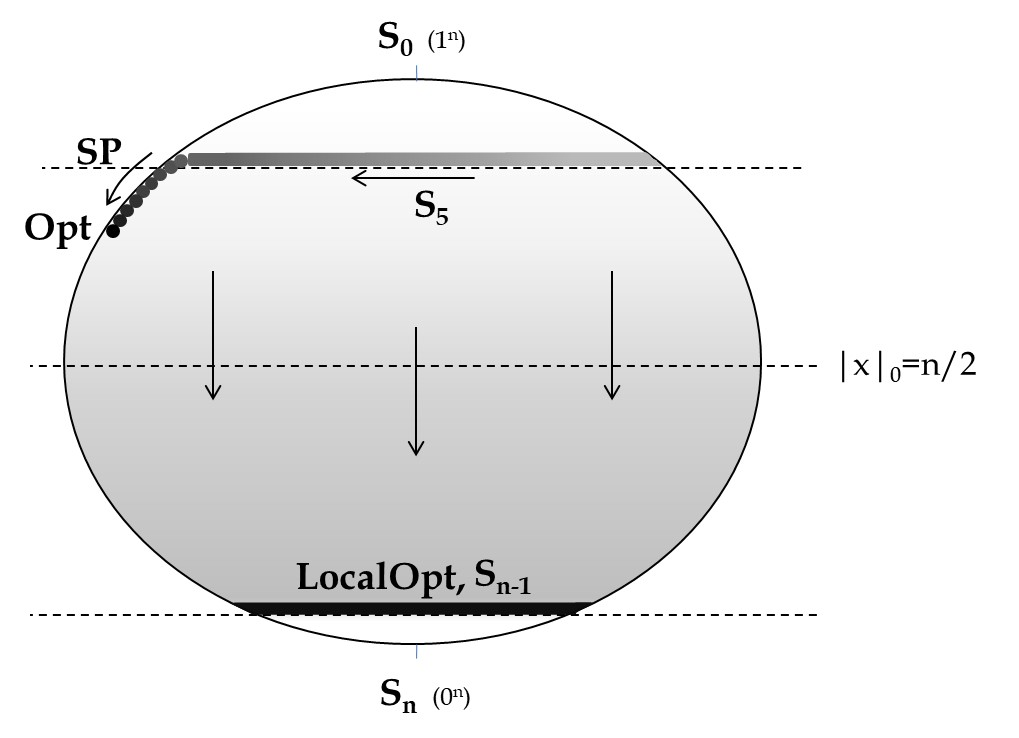}
 \caption{{\color{black} The \textsc{HiddenPath} function (Definition \ref{def:hiddenpath}). The set $S_i$ corresponds to the set of solutions with $i$ 0-bits while the set \textsc{Sp} refers to the set of  solutions in the form $1^{n-k}0^{k}$ for $5\leq k \leq \log{n}+1$. Fitness increases with darker shades of gray.}}
 \label{hp}
 \end{figure}
 
\begin{theorem} \label{th:hiddenpath}
For $c=1$, $dup=1$, $\mu=O(\log n)$ and $\tau=\Omega(n^2 \log n)$, Opt-IA needs expected $O(\tau \mu n+\mu n^{7/2})$ fitness function evaluations to optimise \textsc{HiddenPath}.
\end{theorem}

\begin{proof}
For convenience we will call any solution with $i$ 0-bits (expect the \textsc{Sp} solutions) an $S_i$ solution.
After $O(n \log{n})$ generations in expectation, an $S_{n-1}$ 
solution 
is found by optimising \textsc{ZeroMax}. Assuming the global optimum is not 
found first, consider the generation when an 
$S_{n-1}$ solution is found for the first time. Another $S_{n-1}$ solution is 
created and accepted by Opt-IA with probability at least  $1/n$ since it is 
sufficient to flip the single 1-bit in the first mutation step and any 0-bit 
in the second step will be flipped next with probability 1. Thus, $S_{n-1}$ solutions take over the population in expected 
$O( n \mu)$  generations. Since, apart from the optimum, no other solution has 
higher fitness than $S_{n-1}$ solutions, the 
population consists only of $S_{n-1}$ solutions after the takeover 
occurs. A solution reaches age $\tau$ before the takeover only
with probability $2^{\Omega(\sqrt{n})}$, due to Markov's inequality
applied iteratively for $\Omega(\sqrt{n})$ consecutive phases of length 
$\Theta(n \mu)$. 
We now bound the expected time until {\color{black}the entire} population has the same 
age. Considering that the probability of creating another $S_{n-1}$ 
solution is $\Theta(1/n)$, the probability of creating two copies in one 
single generation is $\binom{\mu}{2} O\left(\frac{1}{n^2}\right)=O \left( 
\frac{\log ^2 
n}{n^2}\right)$. With constant probability this event does not happen in 
$o(n^2/\log^2 n)$ generations. Conditional on that at most one 
additional $S_{n-1}$ solution is created in every generation, we can follow a 
similar argument as in the proof of Theorem~\ref{RLSp}. Hence, we can show that in expected 
$O(\mu^3 n)$ iterations after the takeover, the whole population reaches the 
same age.  
When the population of $S_{n-1}$ solutions with the same age reaches age 
$\tau$, with probability $1/\mu \cdot \left(1-(1/\mu)\right)^{2\mu-1}= \Theta(1/\mu)$ 
a single new clone survives while the rest of the population dies. With 
probability $1-O(1/n)$ the survived clone has hypermutated all $n$ bits 
(i.e., the survived clone is an $S_{1}$ solution). In the following 
generation, the population consists of an $S_1$
solution and $\mu -1$ randomly sampled solutions. {\color{black}With probability 1, the 
$S_1$ solution produces an $S_{5}$ solution via hypermutation. On the other hand, with 
overwhelming probability the randomly sampled solutions still have fitness 
value $n/2 \pm O(\sqrt{n})$, hence the $S_{1}$ solution is removed from the 
population while the $S_{5}$ b-cell is kept. Overall, after $\Theta(\mu)+o(1)$ expected restarts an $S_5$ solution will be found in a total expected runtime of 
$O( \mu \cdot(n \log n + \mu n + \mu^3 n + \tau + 1))= O(\mu \tau)$ generations. We momentarily ignore the event that the $S_5$ solution reaches an 
$S_{n-1}$ point via hypermutation. 

Now, the population 
consists of a single 
$S_5$ solution and $\mu-1$ solutions with at most $n/2 +O(\sqrt{n})$ 
0-bits. We want to bound the expected time until $S_5$ solutions take 
over the population conditional on no $S_{n-1}$ solutions being created. Clones 
of any  $S_5$ solutions are also $S_5$ solutions after hypermutation if one of 
the first two bits to be flipped is a $0$ and the other is a $1$, which happens 
with $O(1/n)$ probability. Moreover, if the outcome of
hypermutation is 
neither an $S_5$, an {\color{black} \textsc{Sp} nor} an $S_{n-1}$ solution, then it is an $S_{n-5}$ 
solution since 
all $n$ bit-flips will have been executed. Since $S_{n-5}$ solutions have 
higher fitness 
value than the randomly sampled solutions, they stay in the population.  In the 
subsequent generation, if hypermutation does not improve an $S_{n-5}$ 
solution (which happens with probability $O(1/n)$), it executes  $n$ 
bit-flips to create yet another $S_{5}$ solution unless the \textsc{Sp} path is found. 

This feedback causes the number of $S_{n-5}$ and 
$S_{5}$ solutions to double in each \mbox{generation} with constant probability 
until they collectively take over the population in $O(\log{\mu})$ generations 
in expectation. 
Then, with constant probability all the $S_{n-5}$ solutions produce an $S_{5}$ 
solution via hypermutation and consequently the population consists 
only of $S_{5}$ solutions. Since the takeover happens in expected 
$O(\log{\mu})$ generations, the probability that it fails to complete in 
$O(\log{\mu})$ generations for consecutive $\Omega(\sqrt{n})$ times is 
exponentially small.  Hence, in $O(\sqrt{n} \log \mu)$ generations w.o.p. (by 
applying Markov's inequality iteratively \cite{OlivetoBookChapter}) {\color{black}the entire 
population} are $S_5$ solutions conditional on $S_{n-4}$ solutions not being created 
before. Since the probability that the static hypermutation 
creates an $S_{n-4}$ solution from an $S_{n-5}$ solution is less than $(4/n) \cdot \mu$, and the probability that it happens in $O(\sqrt{n}\log \mu)$ generations is less than $(4/n) \cdot \mu \cdot O(\sqrt{n}\log \mu)=o(1)$, the takeover occurs with high probability, i.e., $1-o(1)$.

After the whole population consists of only $S_5$ solutions, except 
for the global optimum and the local optima, only other points on the gradient would 
have higher fitness. The probability of improving on the gradient is at least 
$1/n^2$ which is the probability of choosing two specific bits to be flipped (a 
0-bit and a 1-bit) leading towards the b-cell with the best fitness on the 
gradient. Considering that the total number of improvements on the gradient is 
at most $5$, in $O(n^2 \cdot 5)= O(n^2)$ generations in expectation the first 
point of \textsc{Sp} will be found. Now we consider the probability of jumping back to 
the local optima from the gradient, which is 
$\Theta(1)/{\binom{n}{n-4}}=O(n^{-4})$, before finding the first point of the 
\textsc{Sp}.
Due to the {\color{black}Markov's} inequality applied iteratively over 
$\Omega(\sqrt{n})$  consecutive phases of 
length $\Theta(n^2)$ with an appropriate constant, the 
probability that the time to find the first point of \textsc{Sp} is more than 
$\Omega(n^{5/2})$ is less than $2^{-\Omega(\sqrt{n})}$. The probability of 
jumping to a local optimum in $O(n^{5/2})$ steps is at most 
$n^{(-4)}O(n^{5/2})=O(n^{-3/2})$ by the union bound. Therefore, \textsc{Sp} is found 
before any local optima in at most $O(n^{5/2})$ generations with probability 
$1-o(1)$. Hence, the previously excluded event has now been taken into account.}

After $1^{n-5}0^{5}$ is 
added to the population, the best solution on the path is improved with 
probability $\Omega(1/n)$ by hypermutation and in expected $O(n \log {n})$ 
generations the global optimum is found. Since all \textsc{Sp}
and  $S_{5}$ solutions have a Hamming distance smaller than $n-4$ 
and larger than $n/2$ to any $S_{n-1}$ solution, the probability that a 
local optimum  is found before the global optimum is at most $O(n \log{n}) \cdot 
\mu n \binom{n}{4}^{-1}= o(1)$ by the union bound. Thus with probability $1-o(1)$ the time to find the optimum is $O(n \log n)$ generations. 
We pessimistically assume 
that we start over when this event does not occur, which implies that the whole 
process until that point should be repeated {\color{black}$1/(1-o(1))$} times in expectation. Overall the dominating term in the runtime is $O(\tau+ n^{5/2})$ generations. By multiplying this time with 
the maximum possible wasted fitness evaluations per generation ($\mu c n$), the 
upper bound is proven. \end{proof}

In the following two theorems we show that hypermutations and ageing used in conjunction are essential.

\begin{theorem}\label{thm:hard1}
Opt-IA without ageing (i.e., $\tau = \infty$) with $\mu=O(\log{n})$ and 
$dup=O(1)$ cannot 
optimise \textsc{HiddenPath} in less than $n^{\Omega(\log{n})}$ expected fitness function evaluations. 
\end{theorem}

\begin{proof}
The only points with higher fitness than solutions with more than $n/2$ 0-bits are \textsc{Sp} points and $S_5$ points. 
If all solutions in the population have less than $n - 
c_1\log{n}$  0-bits {\color{black}or} less than $n - 
c_1\log{n}$ 1-bits for some $c_1>1$, then the Hamming distance of any solution in the population to any \textsc{Sp} solution  or an $S_5$ solution
is at least $\Omega(\log{n})$.  Since there are no other improving 
solution, the probability that an \textsc{Sp} or an $S_5$ solution will be 
discovered in a single hypermutation operation is  
exponentially small according to the last statement of Lemma~\ref{lem:fcm}. As 
a result, conditional on not seeing these points, the algorithm will find a 
solution with $n-c_1\log n$ 0-bits in $O(\mu \cdot dup\cdot n^2 \log n)$ 
fitness function evaluations by optimising \textsc{ZeroMax} (i.e., Theorem 
\ref{th:optiallfunctiongendive}). In the rest of the proof we will calculate the 
probability of reaching an $S_{n-1}$ solution before finding either an \textsc{Sp} point 
or $S_5$ point, or a complementary solution of an \textsc{Sp} point. The latter 
solutions, if accepted, with high probability would hypermutate into \textsc{Sp} 
points. We call an event \textit{bad} where any of the  mentioned points are 
discovered.

{\color{black}
Any solution in the search space can have at most two Hamming 
neighbours on \textsc{Sp} or at most two Hamming neighbours that their 
complementary bit strings are on \textsc{Sp}. The probability of sampling 
any of these neighbours is in the order of $O(1/n)$ since it is necessary to 
flip a specific bit position. Hamming distance to the remaining \textsc{Sp} points 
and their complementary bit strings are at least two. 
{\color{black}The probability of reaching a particular solution with 
Hamming distance at least two is at most $\binom{n}{2}^{-1}=O(1/n^2)$, 
since two particular bits need to be flipped. Excluding the two Hamming 
neighbours (which are sampled with probability $O(1/n)$), the probability of 
reaching any other \textsc{Sp} point is $O(1/n^2)\cdot O(\log{n})=O(\log{n}/n^2)$ by 
the union bound. Now, taking the Hamming neighbours into consideration}, the 
probability that any of the $\log{n}-3$ 
\textsc{Sp} points (or their complementary bit strings) are discovered is $O(1/n)+O(1/n^2)\cdot O(\log{n})=O(1/n)$. For any initial solution with less than $n-10$ 0-bits, the probability
of finding a solution with five 0-bits is at most $\binom{n}{6}^{-1}\binom{n}{5}= O(1/n)$. Since the probability of reaching an $S_5$ point from solutions with more than $n-10$ 0-bits may be much larger, we first calculate the probability of finding $n-11$ 0-bits before a bad event occurs.

The probability of finding a solution which improves the \textsc{ZeroMax} value is at least $\Theta(1/n)$ (even when we exclude the potential improvements whose complementary bit strings are on \textsc{Sp}). Since at every static hypermutation, the probabilities of finding an $S_5$ solution, an \textsc{Sp} solution or its complementary bit string are all in the order of $O(1/n)$ (up to $n-10$ 0-bits), the conditional probability that a solution which improves the current best \textsc{ZeroMax} value is found before any other improvement  is in the order of $\Theta(1/n)/(\Theta(1/n)+O(dup\cdot\mu/n) )=\Omega(1/(dup\cdot\mu))$. The probability that it  happens $c_1 \log{n}$ times immediately after the first solution with more  than $n-c_1 \log{n}$ 0-bits is added to the 
 population is at least $\Omega(dup\cdot\mu)^{-c_1(\log{n})}=(dup\cdot\mu)^{-O(\log{n})}$. This 
sequence of $c_1 
\log{n}$ improvements implies that a solution with $n-11$ 0-bits is added to the 
population. We now consider the probability that one individual finds the local optimum (i.e., an $S_{n-1}$ solution) by improving its \textsc{ZeroMax} value in 10 consecutive generations and that none of the other individuals improve in these generations. The probability that the current best individual improves in the next step is at least $1/n$ and the probability that the other individuals do not improve is at least $(1-c/n)^{\mu} \geq 1/e$. Hence, the probability that this happens consecutively in the next 10 steps is at least $(1/ne)^{10}=\Omega(n^{-10})$.

Once a locally optimal solution is found, with 
high 
probability it takes over the population before the optimum is found and the 
expected runtime conditional on the current 
population consisting only of  solutions with $n-1$ 0-bits is at least
$\binom{n}{\log{n}} \geq \frac{(n- 
\log{n})^{\log{n}}}{(\log{n})^{\log{n}}}$.
By the law of total expectation, the unconditional expected runtime is lower bounded by
$ E(T) \geq \left(1-o\left(1\right)\right) n^{-10} (dup\cdot\mu)^{-O(\log{n})} 
\frac{(n-\log{n})^{\log{n}}}{(\log{n})^{\log{n}}}$.
This expression is in the order of 
$n^{\Omega (\log{n})}$  for any  $\mu=O(\log{n})$ and $dup= O(1)$.
}
\end{proof}

\begin{theorem}\label{thm:hard2}
With Probability at least $1-e^{-\Omega(n)}$, Opt-IA using SBM and ageing cannot optimise \textsc{HiddenPath} in less than $n^{\Omega(n)}$ fitness function evaluations. 
\end{theorem}

\begin{proof}
{\color{black}
Since the number of 1-bits in an 
initial solution is binomially distributed  with expectation $n/2$ the probability that an initial solution has less than 
$n/3$ 1-bits is bounded above by $e^{-\Omega(n)}$ using Chernoff bounds. Hence, w.o.p., the 
population has a Hamming distance at least  $n/3-\log{n}-1>n/4$ from any 
solution on the path and any solution with five 0-bits 
(${\color{black}\textsc{Sp}}\cup S_5$) by the union bound. Therefore, the 
probability of finding either any of the $\log n$ points on 
{\color{black}\textsc{Sp}} or one of the $\binom{n}{5}$ points on $S_5$  is $\left 
(\log n +\binom{n}{5} \right) \cdot 1/n^{n/4} (1-1/n)^{n-n/4} \leq 
n^{-\Omega(n)}$. 

Since accepting any improvements, except for the ${\color{black}\textsc{Sp}}\cup S_5$
solutions, increases the distance to the ${\color{black}\textsc{Sp}}\cup S_5$ solutions, 
the probability of jumping to an ${\color{black}\textsc{Sp}} \cup S_5$ solution further decreases throughout 
the search process.
}
\end{proof}

%
%

\subsection{When Opt-IA is detrimental} \label{subsec:optiafail}
In this section we present a function class for which Opt-IA is inefficient. The class of functions, which we call \textsc{HyperTrap$_{\gamma}$}, {\color{black}is defined formally in Definition  \ref{def:hypertrap}}. \textsc{HyperTrap$_\gamma$} is inspired by the \textsc{Sp-Target} function introduced in~\cite{NeuSudWit2008} and used in~\cite{Jansen2011} as an example where hypermutations outperform SBM. Compared to \textsc{Sp-Target}, \textsc{HyperTrap$_{\gamma}$} has local and global optima inverted with the purpose of trapping hypermutations. Also there are some other necessary modifications to prevent the algorithm from finding the global optimum via large mutations. The parameter $\gamma$ defines the distance between the local optima and the path to the global optimum.

{\color{black}In \textsc{HyperTrap$_{\gamma}$}, the solutions with $|x|_1 < n/2$ are evaluated by \textsc{OneMax}. Solutions of the form $1^i0^{n-i}$ with $n/2 \leq i \leq n$, {\color{black}shape a short path which is called \textsc{Sp}}. The last point of \textsc{Sp}, i.e., $1^n$ is the global optimum ({\color{black}\textsc{Opt}}). Local optima of \textsc{HyperTrap$_{\gamma}$} {\color{black}(\textsc{LocalOpt})} are formed by the points with $|x|_1 \geq 3n/4$  which have a Hamming distance of at least 
$\gamma n$ to all points in \textsc{Sp}. Also, the points with $|x|_1=n/2$ are ranked among each other such that bit strings with more 1-bits in the beginning have higher fitness. This ranking forms a gradient 
from $0^{n/2}1^{n/2}$, which has the lowest fitness, to $1^{n/2-1}0^{n/2+1}$, 
which has the highest fitness. 
Finally, the fitness of points which do not belong to any of the mentioned 
sub-spaces is evaluated by \textsc{ZeroMax} (i.e., the fitness is $|x|_0$).}


{\color{black}
\begin{definition} \label{def:hypertrap}
Given the definitions of \textsc{Sp}, \textsc{Opt} and \textsc{LocalOpt} as above and {\color{black}$H(x, \textsc{Sp})$} showing the minimum Hamming distance of the individual $x$ to all \textsc{Sp} points, the \textsc{HyperTrap$_\gamma$} function with $0<\gamma\leq 1/8$ is defined for all $x\in\{0,1\}^n$ by

 \[
\textsc{HyperTrap}_{\gamma}(x):= 
\]
 \[
\begin{cases}
		|x|_1 & \text{if} \;\;|x|_1<n/2,\\
        n/2+ \frac{\sum_{i=1}^n (n-i)\cdot(x_i)}{n} & \text{if} \;\; |x|_1=n/2, \\
                n^2 \cdot |x|_1 & \text{if} \;\; x \in \textsc{Sp}:=\{x \mid x=1^i0^{n-i} \;\;\; 
{\color{black}\text{and}} \;\; n/2\leq i<n\},\\
		  n^3  & \text{if}\;\; x \in \textsc{LocalOpt}:=\{x \mid |x|_1 \geq 
3n/4 \;\; {\color{black}\text{and} \;\; H(x,\textsc{Sp})} \geq \gamma n\},\\
      n^4 & \text{if} \;\; x=1^n=\textsc{Opt},\\
      |x|_0 & \text{if} \;\;|x|_1>n/2 \;\; {\color{black}\text{and}} \;\; x \notin 
(\textsc{Sp} \cup \textsc{LocalOpt} \cup \textsc{Opt}).\\      
 \end{cases}
 \label{f}
\]
 \end{definition}
}

 \begin{figure}[t!]
 \centering
  \includegraphics[width=0.6\textwidth]{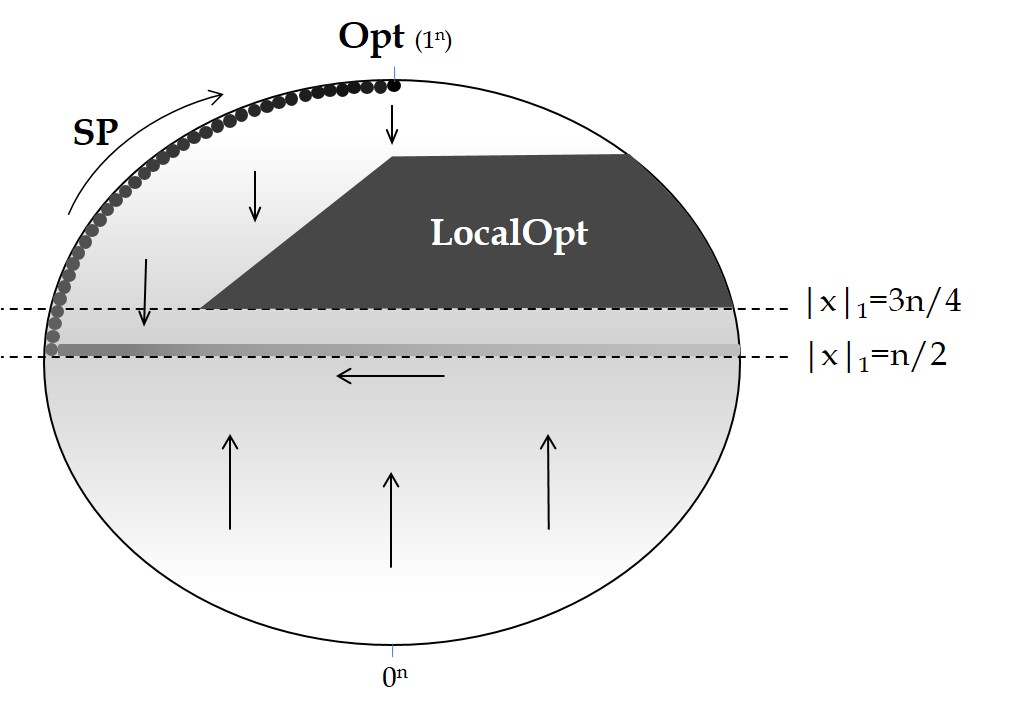}
 \caption{{\color{black}$\textsc{HyperTrap}_{\gamma}$ (Definition \ref{def:hypertrap}). \textsc{Sp} corresponds to the points in the form of $1^i0^{n-i}$ with $n/2 \leq i \leq n$. LocalOpt shows the locally optimal points which have $|x|_1 \geq 3n/4$ and Hamming distance of at least 
$\gamma n$ to all points in \textsc{Sp}. Fitness increases with darker shades of gray.}}
 \label{ht}
 \end{figure}
 
{\color{black}This function is depicted in Figure \ref{ht}}. We will show that there exists a $\textsc{HyperTrap}_{\gamma}$ such that Opt-IA with mutation potential $cn$ gets trapped in the local optima. 
\begin{theorem}
With probability $1-2^{-\Omega(\sqrt{n})}$, Opt-IA with mutation potential $c n$ cannot optimise $\textsc{HyperTrap}_{c/8}$ in less than exponential time.   
\end{theorem}

\begin{proof}
{\color{black}
We will show that the population will first follow the \textsc{ZeroMax} (or \textsc{OneMax})  gradient until it samples a solution with $n/2$ 0-bits and then the gradient of  $\sum_{i=1}^n (n-i)\cdot(x_i)$ until it samples an \textsc{Sp} point with approximately $n/2$ 0-bits. Afterwards, we will prove that large jumps on  \textsc{Sp} are unlikely. Finally, we will show that with overwhelming probability a 
locally optimal solution is sampled before a linear number of \textsc{Sp} points are 
traversed. We optimistically assume that $\tau$ is large enough such that individuals do not die before finding the optimum.  

With overwhelmingly high probability, all randomly  
 initialised individuals have $(1/2 \pm \epsilon)n$ 0-bits for any 
arbitrarily small $\epsilon=\theta(1)$. Starting with such points, the 
probability of jumping to the global optimum is exponentially small according  to Lemma \ref{lem:fcm}. Being optimistic, we assume that the algorithm does not sample a locally optimal point for now. Hence, until an \textsc{Sp} 
point is found, the current best fitness can only be improved if a 
point with either a higher \textsc{ZeroMax} (or \textsc{OneMax})  or   a higher $\sum_{i=1}^n (n-i)\cdot(x_i)$ value than the current best individual is sampled. Since there are less than $n^2$ different values of $\sum_{i=1}^n (n-i)\cdot(x_i)$, and less than $n$ different values of \textsc{ZeroMax} (or \textsc{OneMax}), the current best individual cannot be improved more than $n^2 + n$ times after initialisation without sampling an \textsc{Sp} point.

Let \textsc{Sp$_{lower}$} denote the \textsc{Sp} points with less than $(1/2+ 2 \epsilon)n$ 1-bits, and \textsc{Sp$_{upper}$} denote $\textsc{Sp}\setminus \textsc{Sp}_{lower}$.  We will now show that with overwhelmingly high probability an \textsc{Sp$_{lower}$} point will be sampled before an \textsc{Sp$_{upper}$} point. 
The search points between $(1/2 \pm \epsilon)n$ 1-bits have at least a  distance of $\epsilon n$ to the \textsc{Sp$_{upper}$} points. Using Lemma \ref{lem:fcm}, we can bound the probability of sampling an 
\textsc{Sp$_{upper}$} by $\binom{n}{\epsilon n}^{-1}$ from any input 
solution with $(1/2 \pm \epsilon)n$ 1-bits. Excluding \textsc{Sp} points, a solution with  $(1/2 \pm \epsilon)n$ 1-bits  
has an improvement probability of at least $1/n^2$ (i.e., $\min\{1/n^2, \Theta(1)\}=1/n^2$ with the second term being the probability of improving the \textsc{ZeroMax} or \textsc{OneMax} value 
and the first term the probability of improving when the solution has exactly $n/2$ 0-bits).
Thus, the conditional 
probability that the best solution is improved
before any search point in the population is 
mutated into an \textsc{Sp$_{upper}$} point is  $(n^{-2})/(n^{-2}+ dup \cdot \mu \cdot \binom{n}{\epsilon n}^{-1}) =1- 2^{-\Omega(n)}$.  
This implies that an \textsc{Sp$_{lower}$} point is sampled before an \textsc{Sp$_{upper}$} point with overwhelmingly high probability.
Indeed, by the union bound, this event happens $ n^2 + n$ times consecutively after initialisation  with probability  at least $1-  (n^2+n) \cdot 2^{-\Omega(n)}= 1- 2^{-\Omega(n)}$, 
thus a point in \textsc{Sp$_{lower}$} is sampled before an \textsc{Sp$_{upper}$} point with overwhelmingly high probability.} Each point of \textsc{Sp} has at least $n/2$ 1-bits at the beginning of its bit
string. In order to improve by any value in one mutation operation, 1-bits 
should never be touched before the
mutation operator stops. Hence, the probability of improving by $X$ on
\textsc{Sp} is $p(X)< (1/2)^X$. This yields that the probability of improving by
$\sqrt{n}$ in one generation is at most $dup \cdot \mu \cdot (1/2)^{\sqrt{n}}$, as 
there are $\mu$ individuals in the population. We have now shown that it is 
exponentially unlikely that even an arbitrarily small fraction of the path 
points are avoided by jumping directly to path points with more 1-bits.

 Let $t$ be the first iteration when 
the current \textsc{Sp} solution has at least $99n/100$ 1-bits for the first 
time. We will first lower bound the probability of improving in order to upper 
bound the conditional probability that more than $\sqrt{n}$ new 1-bits are 
added given that an improving path point is sampled. An improvement is achieved 
if the 0-bit adjacent to the last 1-bit is flipped in the first mutation step, 
which happens with probability $1/n$. Thus, the conditional probability of 
improving more than $\Omega(\sqrt{n})$ points on  \textsc{Sp} is at most 
$2^{-\Omega(\sqrt{n})}/(1/n) = 2^{-\Omega(\sqrt{n})}$. Therefore, in generation $t$, 
the current best individual cannot have 
more than $(99n/100)+\Omega(\sqrt{n})$ 1-bits with probability $1-dup \cdot \mu \cdot 2^{-\Omega(\sqrt{n})}$ by the union bound. Similarly, the probability that no jump of size 
$\Omega(\sqrt{n})$  happens in $poly(n)$ generations is at least
$1- dup \cdot \mu \cdot 2^{-\Omega(\sqrt{n})}$.  So, we can conclude that the
number of generations to add another $n/200$ 1-bits to the prefix of the current 
best solution is at least $(n/200)/O(\sqrt{n})= \Omega(\sqrt{n})$ with the same probability. 
%

Now, we
show that during this number of generations, the algorithm gets trapped in the
local optima with overwhelming probability with the optimistic assumption that 
only the best b-cell in the population can be mutated into a locally optimal 
solution. Considering the best current individual, in each generation a 1-bit is 
flipped with probability at least 99/100 in the first mutation step. Thus, the 
probability of not touching a 1-bit in $\Omega(\sqrt{n})$ generations
is less than $(1/100)^{\Omega(\sqrt{n})}$. Considering that $\mu$ and $dup$ are both $poly(n)$, the
probability of flipping a 1-bit from the prefix in this number of generations and of not 
locating the optimum in the same amount of time is
{\color{black}
\begin{align*} P \geq \left(1- \left(\frac{1}{100}\right) ^{\Omega(\sqrt{n})}\right)\cdot \left(1- \frac{\sqrt{n}}{200} \cdot dup \cdot  \mu \cdot 2^ 
{-\Omega(\sqrt{n})} \right) \geq 1-2^{-\Omega(\sqrt{n})}.
\end{align*}
} 
After flipping a 1-bit, the mutation operator mutates at most $c n$ bits until it finds an improvement. All sampled solutions in the first $n/4-n/100$ mutation steps will have more than $3n/4$  1-bits and thus satisfy  the first condition of the local optima. If the Hamming distance between one of the first $n/4-n/100$ sampled solutions and
all \textsc{Sp} points is at least $\gamma n$, then the algorithm is in the trap. We only consider the \textsc{Sp} points with a prefix 
containing more than $3n/4-\gamma n$ 1-bits since \textsc{Sp} points with less 
than $3n/4- \gamma n$ 1-bits have already Hamming distance more than $\gamma n$ 
to the local optima. 


Similarly to the analysis done in~\cite{Jansen2011}, we consider the $k$th mutation step for $k:=4\gamma n (1+1/5)/(3-4\gamma)= \frac{6 c n}{30-5 c} \leq cn$ where the 
last expression is due to $\gamma=c/8$. After $k$ steps, the expected 
number of bits flipped in the prefix of length  $n(3/4-\gamma) $ is at least 
$k(3/4-\gamma)= (1+1/5) \gamma n$. For any mutation potential $0<c\leq 1$, $\frac{6 c n}{30-5 c} \leq n/4- n/100$.  
Using Chernoff bounds, we can show that the probability of having less than $\gamma n$ 0-bits in the prefix is 
 $P(X \leq (1-0.2)E(X)) \leq e^{\frac{-\gamma n \cdot 6}{1000}}=e^{-\Omega(n)}$.

Altogether, with probability $(1-e^{-\Omega(\gamma n)}) \cdot (1-2^{-\Omega(\sqrt{n})})=1-e^{-\Omega(\sqrt{n})}$ a point in the local optima is sampled. 
In $O(\log \mu)=O(\log n)$ generations this individual takes over the 
population. Once in the local optima, the algorithm needs at least $dup \cdot 
\mu \cdot 
1/\binom{n}{\gamma n}$ time to find the global optimum {\color{black}according to Lemma \ref{lem:fcm}}. 
\end{proof}
{\color{black}In the following theorem we see how the \oneoneea using SBM optimises \textsc{HyperTrap$_{\gamma}$} in polynomial time.}
\begin{theorem}
The \oneoneea optimises \textsc{HyperTrap$_{\gamma}$} in $O(n^{3+\epsilon})$ steps w.o.p. $1-e^{-\Omega(n^{\epsilon})}$ for $\epsilon >0$.
\end{theorem}

\begin{proof}
With overwhelming probability, the algorithm is initialised within $n/2 \pm 
n/10$ 1-bits by Chernoff bounds. Since the distance to any trap point is 
linear and the probability that SBM flips at least $\sqrt{n}$ bits in a 
single iteration ($\binom{n}{\sqrt{n}}n^{-\sqrt{n}}=2^{-\Omega(\sqrt{n})}$) is 
exponentially small, so is the probability of mutating into a trap point. As 
the algorithm approaches the bit string with $n/2$ 1-bits, the distance to the 
trap remains linear. Conditional on not entering the trap, by standard AFL arguments it 
takes the \oneoneea at most $O(\sqrt{n})$ steps in expectation to find a 
point with $n/2$ 1-bits, i.e., on the gradient. After finding such a point, 
the \oneoneea improves on the gradient with probability at least {\color{black}$1/n^2 \cdot 
(1-1/n)^{n-2} \geq 1/(en^2)$} which is the probability of flipping the rightmost 1-bit 
and the leftmost 0-bit while leaving the rest of the bits untouched. To reach the first 
point of \textsc{Sp}, there are a linear number of 1-bits that need to be shifted 
to the beginning of the bit string. It therefore takes $O(n^3)$ steps in expectation to find \textsc{Sp}. The 
\oneoneea improves on \textsc{Sp} with probability at least $1/n \cdot 
(1-1/n)^{n-1} \geq 1/en$, which is the probability of flipping the leftmost 
0-bit and not touching the other bits. Hence, the global 
optimum is found in $O(n^2)$ steps in expectation giving a total expected 
runtime of $O(n^3)$ conditional on not falling into the trap. By applying 
Markov's inequality iteratively over $\Omega(n^{\epsilon})$ consecutive phases 
of length $\Theta(n^3)$ with an appropriate constant, the probability that the 
optimum is not found within $O(n^{3+\epsilon})$ steps is less than 
$e^{-n^\epsilon}$ with $\epsilon$ being an arbitrarily small constant. Since the 
probability of finding a local {\color{black}optimum} from the gradient points or 
\textsc{Sp} points is $n^{-\Omega(n)}$ in each step, the probability of not 
falling into the trap in $n^{3+\epsilon}$ steps is less than 
$(n^{3+\epsilon}\cdot n^{-c'n})$ by the union bound. Overall, the total 
probability of finding the optimum within $O(n^{3+\epsilon})$ steps is bigger 
than $(1- e^{-n^\epsilon})\cdot(1-n^{3+\epsilon}\cdot n^{-c'n})=1- 
e^{-\Omega(n^\epsilon)}$.
\end{proof}

\subsection{On Trap Functions} \label{subsec:trap}
In~\cite{optia}, where Opt-IA was originally introduced, the effectiveness of the algorithm was tested for optimising the following simple trap function:
\begin{align*}
\textsc{Simple Trap}(x):=
\begin{cases} 
	\frac{a}{z}(z-|x|_1) & \text{if} \; |x|_1\leq z,\\
    \frac{b}{n-z}(|x|_1-z) & \text{otherwise},
\end{cases}
\end{align*}
where $z \approx n/4$, $b=n-z-1$, $3b/2 \leq a \leq 2b$ and the optimal solution is the $0^n$ bit string.

{\color{black}The reported experimental results were averaged over 100 independent runs each with a termination criterion of reaching $5 \times 10^5$ fitness function evaluations. For all of the results, the population size (i.e., $\mu$) was 10 and $dup=1$. In these experiments, Opt-IA using either hypermutations or 
hypermacromutation never find the optimum of  \textsc{Simple Trap} already for 
problem sizes $n> 50$.  However, the following theorem shows that Opt-IA$^*$ indeed optimises \textsc{Simple Trap} efficiently.
}

\begin{theorem}\label{thm:trap}
{\color{black}Opt-IA} needs $O(\mu n^2 \log n)$ expected fitness function 
evaluations to 
optimise \textsc{Simple Trap} with $\tau=\Omega(n^{2})$ for $c=1$ and $dup=1$.
\end{theorem}

\begin{proof}

Given that the number of 1-bits in the 
best solution is $i$, the probability of improving is at least $(n-i)/n$ if the 
best solution has more than $z$ 1-bits and $i/n$ otherwise. By following the 
proof of Theorem \ref{onemax}, at least one individual will reach $1^n$ or 
$0^n$ in $ O(\mu n^2 \log n)$ fitness function evaluations in expectation as 
long as no individual dies due to ageing.

The age of an individual reaches $n^2$ only if the improvement fails to 
happen in $n^2$ generations which happens with probability at most 
$(1-1/n)^{n^2}=2^{-\Omega(n)}$ since the improvement probability is at least 
$1/n$. This implies that the expected number of restarts by ageing is 
exponentially small.
%
%
%
%
\end{proof}

Given that Opt-IA was tested in~\cite{optia} also with the parameters suggested 
by Theorem \ref{thm:trap} (i.e., $c=1$, $dup=1$, $\tau=\infty$),  
we speculate that either FCM was mistakingly not used or the stopping criterion 
(i.e., the total number of allowed fitness evaluations, i.e., $5 \times 10^5$) 
was too small.
We point out that, for large enough $\tau$ also using hypermacromutation as mutation operator would lead to an expected 
runtime of $O(\mu n^2 \log n)$ for \textsc{Trap} functions~\cite{JansenZarges2011a}, with 
or without FCM. 
In any case, it is not necessary to apply both hypermutations and 
hypermacromutation together to efficiently optimise \textsc{Trap} functions as reported in~\cite{optia}. On the other hand, the 
inversely proportional hypermutation operator considered in~\cite{Jansen2011} 
would fail to optimise this function efficiently because it cannot flip $n$ bits when on the local optimum.

\section{Not Allowing Genotype Duplicates} \label{sec:diversity}

None of the algorithms considered in the previous sections use the genotype diversity mechanism. In this section, we do not allow genotype 
redundancies in the population as proposed in the original algorithm~\cite{optia,optia-transaction}. 
This change potentially affects the behaviour of the algorithm. {\color{black}In the following, we will first consider the ageing operator in isolation with genotype diversity (i.e., no genotypic duplicates are allowed in the population). Afterwards we will analyse the complete Opt-IA algorithm with the same diversity mechanism (as originally introduced in the literature).}

\subsection{\realmuonerlsageing with genotype diversity}

In this subsection, we analyse  \realmuonerlsageing (Algorithm \ref{realrlsageing}) with   genotype diversity for optimising \textsc{Cliff$_d$} for which \muonerlsageing without genotype 
diversity was previously analysed in Section \ref{sec:ageing}. 
The main difference compared to the analysis there is that taking over the population on the local optima is now harder since identical individuals are not allowed. 
The proof of the following theorem shows that the algorithm can still take over and use ageing to escape from the local optima as long as the population size is not too large. 

\begin{algorithm}[t]
\begin{algorithmic}[1]
\STATE{{\color{black}$t=0$},}
 \STATE{{\color{black}initialise $P^{(t)}=\{x_1,...,x_\mu\}$, a population of $\mu$ individuals uniformly at random and set $x_i^{age}=0$ for $i=\{1,...\mu\}$}.}
 \WHILE {the optimum is not found}
\STATE{ select $x\in P^{(t)}$ uniformly at random,}
\STATE{ create $y$ by flipping one bit of $x$,}
\STATE{Hybrid ageing~$(P^{(t)}, \{y\}, \tau, \mu)$,}
\STATE{Selection~($P^{(t)},\{y\}, \mu, 1$),} 
\STATE{$t=t+1$.}
\ENDWHILE
  \end{algorithmic}
  \caption{\realmuonerlsageing with genotype diversity (i.e., $div=1$).}
 \label{realrlsageing}
\end{algorithm}

\begin{theorem}
For {\color{black}constant} $\mu$ and $\tau=\Theta(n)$, the \realmuonerlsageing  with
genotype diversity optimises \textsc{Cliff$_d$} 
in expected $O(n \log n)$  fitness function evaluations for any linear $d \leq n/4 -\epsilon$.
\end{theorem}

\begin{proof}

{\color{black}By Chernoff bounds}, with overwhelming probability the initial individuals are sampled with $n (1/2 \pm \epsilon)$ 1-bits  for any arbitrarily small $0<\epsilon=\Theta(1)$.  Since the population size is constant and there is a constant probability of improving in the first mutation step, a local optimum is found in at most $O(n)$ steps by picking the best individual and improving it $O(n)$ times.

  If there is a single locally optimal solution in the population, then with probability $1/\mu$  this individual is selected as parent and 
produces an offspring with fitness $n-d-1$ (i.e., one bit away) with  probability 
$(n-d-i)/n \leq (n-d-\mu)/n =\Theta(1)$ where $i$ is the number of individuals already with fitness $n-d-1$. 
In the next generation, with probability $(i+1)/\mu $ one of the individuals on the local optimum or one step away from it is selected as parent and produces 
an offspring on either the local optimum or one bit away with probability at least $(n-d-\mu)/n =\Theta(1)$. Hence, in expected time $\mu 
\cdot \Theta(1)= \Theta(\mu)$ all $\mu$ individuals are  on the local optimum. 

When the last inferior solution is replaced by an individual on the local optimum, the other individuals have ages in the order of $O(\mu)$. 
Thus, the probability that the rest of the population does not die until the youngest individual reaches age $\tau$, is at least $(1/\mu)^{({\mu-1})\cdot O(\mu)}=\Theta(1)$, 
the probability that $\mu-1$ individuals above age $\tau$ survive $O(\mu)$ consecutive generations.  
 
In the first generation when the last individual reaches age $\tau$, with probability $d/n$ an offspring is created at the bottom of the cliff (i.e., with a fitness value of $n-d+1$) and with
probability $1/\mu \cdot (1-1/\mu)^{\mu} = \Theta(1)$ all the parents die together at that step 
and the offspring survives. The rest of the proof follows the same arguments as the proof of Theorem \ref{RLSp}. 

Overall, the total expected time to optimise \textsc{Cliff$_d$}  is dominated by the time 
to climb the second \textsc{OneMax} slope which takes $O(n \log n)$ steps in expectation. 
\end{proof}

\subsection{Opt-IA with genotype diversity}

In this subsection, we analyse Opt-IA (Algorithm \ref{modified-optia}) with genotype diversity to optimise all the functions for which Opt-IA without genotype diversity was analysed in Section \ref{sec:optia}. The following are straightforward corollaries of Theorem \ref{optiaallfunc} and Theorem \ref{thm:trap}. Since the ageing mechanism never triggers here and the proofs of those theorems do not depend on creating genotype copies, the arguments are still valid for Opt-IA  with genotype diversity. 
\begin{corollary}
The upper bounds on the expected runtime of Opt-IA with genotype diversity and ageing parameter 
$\tau$ large enough for \textsc{OneMax}, \textsc{LeadingOnes}, $\textsc{Jump}_k$ and $\textsc{Cliff}_d$ are as follows:\\ $
 E(T_{\textsc{OneMax}}) =O\left(\mu \cdot dup \cdot  n^2 \log{n}\right),\;
 \\E(T_{\textsc{LeadingOnes}}) =O\left(\mu \cdot dup \cdot  n^3\right),\;
 \\E(T_{\textsc{Jump}_k}) =O\left(\mu \cdot dup \cdot \frac{n^{k+1} \cdot 
e^k}{k^k}\right),\;
\\ \text{and}\;  E(T_{\textsc{Cliff}_d}) = O\left(\mu\cdot dup \cdot \frac{n^{d+1} \cdot e^d}{d^d}\right).\;
 $ 
\end{corollary}

\begin{corollary}
 Opt-IA with genotype diversity needs $O\left(\mu n^2 \log n\right)$ expected fitness function evaluations to optimise \textsc{Simple Trap} with $\tau=\Omega(\mu n^{1+\epsilon})$, $c=1$ and $dup=1$.
\end{corollary}

The following theorem shows the same expected runtime for Opt-IA with genotype diversity for \textsc{HiddenPath} as that of the Opt-IA without genotype diversity proven 
in {\color{black}Theorem~\ref{th:hiddenpath}}.  However, we reduce the population size to be constant. 
The proof follows the main arguments of the proof of Theorem \ref{th:hiddenpath}. Here we only discuss the probability of events which should be handled differently considering that genotype duplicates are not allowed. 

\begin{theorem}
For $c=1$, $dup=1$, $\mu=\Theta(1)$ and $\tau=\Omega(n^2 \log n)$, Opt-IA with genotype diversity needs $O(\tau \mu n+\mu n^{7/2})$ expected fitness function evaluations to optimise \textsc{HiddenPath}.
\end{theorem}

\begin{proof}
We follow the analysis of the proof of Theorem \ref{th:hiddenpath} for Opt-IA without genotype diversity. 
Although the analysis did not benefit from genotype duplicates, not allowing them potentially affects the runtime of the events where the population takes over.
The potentially affected events are:
\begin{itemize}

\item For $S_{n-1}$ solutions to take over the population, the probabilities are different here since a new $S_{n-1}$ solution will not be accepted if it is identical to any current $S_{n-1}$ solutions. Here, after finding the first $S_{n-1}$ solution, the rest are created and accepted by Opt-IA with probability at least $\Omega(1/n \cdot (n-\mu)/(n-1))=\Omega(1/n)$. Therefore, the argument made in the proof of Theorem \ref{th:hiddenpath} does not change the runtime asymptotically. 

\item The arguments about the expected time needed for $S_{n-1}$ solutions to reach the same age after the takeover are the same as in the proof of Theorem \ref{th:hiddenpath}; 
without genotype duplicates, the probability of creating another $S_{n-1}$ is still $\Omega(1/n)$. Hence, the probability of creating two copies in the same generation is still 
unlikely and we can adapt the arguments made in the proof of Theorem \ref{th:hiddenpath}. Therefore, in expected $O(\mu^3n)$ generations after the takeover of $S_{n-1}$, the population 
reaches the same age. 

\item  In the proof of Theorem \ref{th:hiddenpath}, the expected time for $S_5$ solutions to take over the population of recently initialised individuals is bounded relying on having multiple copies of one $S_5$ and one $S_{n-5}$ solution.  This proof strategy cannot be applied with the genotypic diversity mechanism which only allows unique solutions in the population.

{\color{black} 
In order to create a unique $S_5$ solution from another $S_5$ solution, it is sufficient that the first bit position to be flipped has value $1$ in the parent bit string and the second position to be flipped has value 0 in all $S_5$ solutions currently in the population, including the parent. Such a mutation occurs with probability at least  $(\frac{5}{n}\cdot \frac{n-5-\mu}{n-1}) \geq 4/n=\Omega(1/n)$. This lower bound in the probability implies that the number of $S_5$ individuals in the population is increased by a factor of $1+\Omega(1/n)$ at every generation in expectation. For any constant $c$, the exponent $t$ that satisfies $1\cdot(1+c/n)^t=\mu$ is in the order of $O(n\log{\mu})$. Thus, in at most $O(n \log{\mu})$ generations in expectation the $S_5$ solutions take over the population if no solutions with higher fitness have been added to the population before the takeover.  Due to the assumption that $\mu=\Theta(1)$, the expected number of generations reduces to $O(n)$. By Markov's inequality, the probability that the expected time is in the order of $O(n)$ is $\Omega(1)$.

Only the solutions on \textsc{Sp} and $S_{n-1}$ solutions have better fitness than $S_5$ solutions. The rest of the proof of Theorem~\ref{th:hiddenpath} can still be applied if the only remaining non $S_5$ solutions in the population are \textsc{Sp} solutions. So, we will only show that it is unlikely that an $S_{n-1}$ solution will be sampled before the population consists only of $S_{5}$ and \textsc{Sp} solutions. The number of 0-bits in randomly initialised solutions is in the interval of  $[n/2 -n^{2/3}, n/2 +n^{2/3}]$ with probability $1-2^{\Omega(n^{1/3})}$ due to Chernoff bounds. We can then follow the strategy from Theorem~\ref{thm:lbonemax} to show that the expectation of $T^*$, the time until an $S_{n-1}$ solution descending from  a randomly created solution is sampled, is in the order of $\Omega(n \log{n})$. In order to bound the probability that this event will not happen in $O(n)$, we will bound the variance of this runtime. Since FCM stops after a solution with more $0$-bits is sampled,  we can pessimistically divide the runtime into phases of length $T_i$, $i\in \{2,\ldots, n/2+n^{2/3}\}$, where the best among the newly created solutions has $i$ 1-bits. The length of phase  $T_i$ is distributed according to a geometric distribution with success probability at most  $\mu\cdot 2 i/n$ due to the Ballot theorem (i.e., Theorem \ref{thm:ballot}) and the union bound summed up over $\mu$ candidate solutions which can be improved at every generation. Being geometrically distributed, the variance of $T_i$ is at most $(1-(d \cdot i/ n))/(d\cdot i /n)^2$ for some constant $d  > 2 \mu$. Summing up over all phases of independently distributed lengths, we obtain the variance of $T^*$ as:
\begin{align*}
 Var(T^*) \leq \sum\limits_{i=2}^{n/2+n^{2/3}} \frac{1-(d\cdot i/ n)}{(d\cdot i /n)^2} \leq \frac{n^2}{d^2}\sum\limits_{i=2}^{n/2+n^{2/3}}\frac{1}{i ^2} =\frac{n^2}{d^2}\frac{\pi^2}{6}=\Theta(n^2).
\end{align*}
 Due to Chebyshev's Inequality, the probability that such an event happens in $O(n)$ generations instead of its expectation, which is in the order of $\Omega(n\log{n})$, is at most $O(1/\log{n})$. Conversely, the probability that such a failure does not occur is $1-O(1/\log{n})$.

The path solutions have between $5$ to $\log{n}+1$ 1-bits, thus the probability that the hypermutation operator yields an $S_{n-1}$ solution as output given an \textsc{Sp} or $S_{5}$ solution as input is at most $\binom{n}{n-4}^{-1}=O(n^{-4})$. The probability that such a mutation occurs in $O(n)$ generations is at most $O(n^{-3})$ by the union bound. Considering all the possible failure events, the takeover happens in $O(n)$ generations with $\Omega(1)$ probability before any $S_{n-1}$ individual is added to the population.
}
\end{itemize}

The rest of the proof of Theorem \ref{th:hiddenpath} is not affected by genotype diversity. 
\end{proof}

\section{Conclusion} \label{sec:conclusion}
We have presented an analysis of the standard Opt-IA artificial immune system. 
We first highlighted how both the ageing and hypermutation operators may allow 
to efficiently escape local optima that are particularly hard for standard 
evolutionary algorithms. Concerning hypermutations, we proved that FCM is essential to the 
operator and suggested  considering a mutation \emph{constructive} if the 
produced fitness is at least as good as the previous one. The reason is that far away points of equal fitness should be attractive for 
the sake of increasing exploration capabilities. Our analysis on the \textsc{Jump$_k$} function suggests that hypermutation with FCM is generally preferable to the SBM operator when escaping the local optima requires a jump of size $k$ such that $(k/e)^k \ge n$. This advantage is least pronounced when there is a single solution with better fitness than the local optima as in the case of \textsc{Jump$_k$} and the performance difference with the SBM in terms of expected escape time scales multiplicatively with the number of acceptable solutions at distance $k$. Hence, the \textsc{Jump$_k$} function may be considered as a worst-case scenario concerning the advantages of hypermutations over SBM for escaping local optima.

Concerning ageing, we showed for the first time 
that the operator can be very efficient when coupled with {\color{black}SBM} 
and hypermutations. To the best of our knowledge, the operator  allows the best 
known expected runtime (i.e., $O(n\log{n})$) for hard \textsc{Cliff$_d$} functions (this expected runtime has recently been matched by a simple hyperheuristic \cite{LissovoiOlivetoWarwicker2019}).
Afterwards, we presented  a class of functions where both the characteristics of 
ageing and hypermutation are crucial, hence  Opt-IA is efficient while  standard 
evolutionary algorithms are inefficient even if coupled with one extra AIS 
operator (either cloning, ageing, hypermutation or contiguous somatic mutation). 
Finally, we proved that all the positive results presented for the Opt-IA algorithm
without genotype diversity also hold for Opt-IA with genotype 
diversity as used in the original algorithm. However, small population sizes may 
be required if ageing has to be triggered to escape from local optima with 
genotype diversity.
To complete the picture we presented a class of problems where the use of hypermutations and ageing is detrimental while standard evolutionary algorithms are efficient.

Our analysis shows that for easy problems for which local search strategies are efficient, using hypermutations and ageing may be detrimental.
We have shown this effect for the simple \textsc{OneMax} and \textsc{LeadingOnes} functions for which we have proven a linear slow-down in the expected runtime compared to local search 
strategies and simple evolutionary algorithms. While such a slow-down may be considered an acceptable expense in exchange for being efficient on more complicated optimisation problems, 
we have shown for \textsc{HyperTrap$_{c/8}$} that the consequences may be more drastic, by making the difference between polynomial
and exponential runtimes. Indeed, the function is easy for local search 
algorithms with neighbourhoods of size greater than 1 and for simple EAs.
However, hypermutations make Opt-IA fail with overwhelming probability and ageing does not help the algorithm to escape from the trap permanently. We point out that recently {\it Fast Hypermutation} operators have been presented that allow to hillclimb efficiently while still keeping the efficiency at escaping local optima \cite{CorusOlivetoYazdaniFASTAIS20018}. However, such Fast AIS still suffer on \textsc{HyperTrap$_{c/8}$}. 

On the other hand, for more complicated multimodal functions, with closer characteristics to the optimisation problems that occur in practice, 
we have shown several advantages of hypermutations and ageing to escape from local optima (i.e., \textsc{Jump$_k$} and \textsc{Cliff$_d$}).
Furthermore, the combination of hypermutation and ageing may be advantageous to locate 
new promising basins of attraction that are hard to find via more traditional optimisation techniques such as local search or EAs using SBM.
We have illustrated such effect in the analysis of \textsc{HiddenPath}, where the combination of hypermutations and ageing allow Opt-IA to locate a new basin of attraction that 
initially has lower fitness than the easy-to-find local optima. However, in the long run, having identified this new basin of attraction allows the algorithm to find the global optimum.

Overall, we believe this work is a significant contribution towards the understanding of which kind of problems it is advantageous to use artificial immune systems on rather than 
evolutionary algorithms and for which it is detrimental. Future work should focus on providing such advantages and disadvantages for classical combinatorial optimisation problems, in similar fashion to the recently presented results for the NP-hard \textsc{Partition} problem that have been achieved by building upon the present work \cite{CorusOlivetoYazdaniPartition2018}.

\smallskip
{\color{black}\textbf{Acknowledgement:}}
The research leading to these results has received funding from the EPSRC under 
grant agreement no EP/M004252/1.

\section*{References}

\bibliography{mybibfile}

\begin{thebibliography}{40}
\expandafter\ifx\csname natexlab\endcsname\relax\def\natexlab#1{#1}\fi
\providecommand{\url}[1]{\texttt{#1}}
\providecommand{\href}[2]{#2}
\providecommand{\path}[1]{#1}
\providecommand{\DOIprefix}{doi:}
\providecommand{\ArXivprefix}{arXiv:}
\providecommand{\URLprefix}{URL: }
\providecommand{\Pubmedprefix}{pmid:}
\providecommand{\doi}[1]{\href{http://dx.doi.org/#1}{\path{#1}}}
\providecommand{\Pubmed}[1]{\href{pmid:#1}{\path{#1}}}
\providecommand{\bibinfo}[2]{#2}
\ifx\xfnm\relax \def\xfnm[#1]{\unskip,\space#1}\fi
\bibitem[{Corus et~al.(2017)Corus, Oliveto, and
  Yazdani}]{CorusOlivetoYazdani2017}
\bibinfo{author}{D.~Corus}, \bibinfo{author}{P.~S. Oliveto},
  \bibinfo{author}{D.~Yazdani},
\newblock \bibinfo{title}{On the runtime analysis of the {Opt-IA} artificial
  immune system},
\newblock in: \bibinfo{booktitle}{Proc. of GECCO 2017}, \bibinfo{year}{2017},
  pp. \bibinfo{pages}{83--90}.
\bibitem[{de~Castro and Timmis(2002)}]{decastrobook}
\bibinfo{author}{L.~N. de~Castro}, \bibinfo{author}{J.~Timmis},
  \bibinfo{title}{Artificial Immune Systems: A New Computational Intelligence
  Paradigm}, \bibinfo{publisher}{Springer-Verlag}, \bibinfo{address}{Secaucus,
  NJ, USA}, \bibinfo{year}{2002}.
\bibitem[{Burnet(1959)}]{Burnet1959}
\bibinfo{author}{F.~M. Burnet}, \bibinfo{title}{The Clonal Selection Theory of
  Acquired Immunity}, \bibinfo{publisher}{Cambridge University Press},
  \bibinfo{year}{1959}.
\bibitem[{de~Castro and Zuben(2002)}]{decastro}
\bibinfo{author}{L.~N. de~Castro}, \bibinfo{author}{F.~J.~V. Zuben},
\newblock \bibinfo{title}{Learning and optimization using the clonal selection
  principle},
\newblock \bibinfo{journal}{IEEE Transaction on Evolutionary Computation}
  \bibinfo{volume}{6} (\bibinfo{year}{2002}) \bibinfo{pages}{239--251}.
\bibitem[{Kelsey and Timmis(2003)}]{kelsey}
\bibinfo{author}{J.~Kelsey}, \bibinfo{author}{J.~Timmis},
\newblock \bibinfo{title}{Immune inspired somatic contiguous hypermutation for
  function optimisation},
\newblock in: \bibinfo{booktitle}{Proc. of GECCO 2003}, \bibinfo{year}{2003},
  pp. \bibinfo{pages}{207--218}.
\bibitem[{Cutello et~al.(2006)Cutello, Pavone, and Timmis}]{optia-transaction}
\bibinfo{author}{V.~Cutello}, \bibinfo{author}{M.~Pavone},
  \bibinfo{author}{J.~Timmis},
\newblock \bibinfo{title}{An immune algorithm for protein structure prediction
  on lattice models},
\newblock \bibinfo{journal}{IEEE Transactions on Evolutionary Computation}
  \bibinfo{volume}{10} (\bibinfo{year}{2006}) \bibinfo{pages}{844--861}.
\bibitem[{Cutello et~al.(2007)Cutello, Nicosia, Romeo, and
  Oliveto}]{Oliveto2007a}
\bibinfo{author}{V.~Cutello}, \bibinfo{author}{G.~Nicosia},
  \bibinfo{author}{M.~Romeo}, \bibinfo{author}{P.~S. Oliveto},
\newblock \bibinfo{title}{On the convergence of immune algorithms},
\newblock in: \bibinfo{booktitle}{Proc. of FOCI 2007}, \bibinfo{year}{2007},
  pp. \bibinfo{pages}{409--415}.
\bibitem[{Jansen and Zarges(2011)}]{JansenZarges2011a}
\bibinfo{author}{T.~Jansen}, \bibinfo{author}{C.~Zarges},
\newblock \bibinfo{title}{Analyzing different variants of immune inspired
  somatic contiguous hypermutations},
\newblock \bibinfo{journal}{Theoretical Computer Science} \bibinfo{volume}{412}
  (\bibinfo{year}{2011}) \bibinfo{pages}{517--533}.
\bibitem[{Corus et~al.(2016)Corus, He, Jansen, Oliveto, Sudholt, and
  Zarges}]{EasiestFunctions}
\bibinfo{author}{D.~Corus}, \bibinfo{author}{J.~He},
  \bibinfo{author}{T.~Jansen}, \bibinfo{author}{P.~S. Oliveto},
  \bibinfo{author}{D.~Sudholt}, \bibinfo{author}{C.~Zarges},
\newblock \bibinfo{title}{On easiest functions for mutation operators in
  bio-inspired optimisation},
\newblock \bibinfo{journal}{Algorithmica}  (\bibinfo{year}{2016}).
\bibitem[{Zarges(2008)}]{Zarges2008}
\bibinfo{author}{C.~Zarges},
\newblock \bibinfo{title}{Rigorous runtime analysis of inversely fitness
  proportional mutation rates},
\newblock in: \bibinfo{booktitle}{Proc. of PPSN~X}, \bibinfo{year}{2008}, pp.
  \bibinfo{pages}{112--122}.
\bibitem[{Zarges(2009)}]{Zarges2009}
\bibinfo{author}{C.~Zarges},
\newblock \bibinfo{title}{On the utility of the population size for inversely
  fitness proportional mutation rates},
\newblock in: \bibinfo{booktitle}{Proc. of FOGA 2009}, \bibinfo{year}{2009},
  pp. \bibinfo{pages}{39--46}.
\bibitem[{Jansen and Zarges(2011)}]{JansenZarges2011c}
\bibinfo{author}{T.~Jansen}, \bibinfo{author}{C.~Zarges},
\newblock \bibinfo{title}{On the role of age diversity for effective aging
  operators},
\newblock \bibinfo{journal}{Evolutionary Intelligence} \bibinfo{volume}{4}
  (\bibinfo{year}{2011}) \bibinfo{pages}{99--125}.
\bibitem[{Horoba et~al.(2009)Horoba, Jansen, and Zarges}]{HorobaJansenZarges09}
\bibinfo{author}{C.~Horoba}, \bibinfo{author}{T.~Jansen},
  \bibinfo{author}{C.~Zarges},
\newblock \bibinfo{title}{Maximal age in randomized search heuristics with
  aging},
\newblock in: \bibinfo{booktitle}{Proc. of GECCO 2009}, \bibinfo{year}{2009},
  pp. \bibinfo{pages}{803--810}.
\bibitem[{Oliveto and Sudholt(2014)}]{ageing}
\bibinfo{author}{P.~S. Oliveto}, \bibinfo{author}{D.~Sudholt},
\newblock \bibinfo{title}{On the runtime analysis of stochastic ageing
  mechanisms},
\newblock in: \bibinfo{booktitle}{Proc. of GECCO 2014}, \bibinfo{year}{2014},
  pp. \bibinfo{pages}{113--120}.
\bibitem[{Jansen et~al.(2011)Jansen, Oliveto, and
  Zarges}]{JansenOlivetoZarges2011}
\bibinfo{author}{T.~Jansen}, \bibinfo{author}{P.~S. Oliveto},
  \bibinfo{author}{C.~Zarges},
\newblock \bibinfo{title}{On the analysis of the immune-inspired {B-Cell}
  algorithm for the vertex cover problem},
\newblock in: \bibinfo{booktitle}{Proc. of ICARIS 2011}, \bibinfo{year}{2011},
  pp. \bibinfo{pages}{117--131}.
\bibitem[{Jansen and Zarges(2012)}]{JansenZarges2012}
\bibinfo{author}{T.~Jansen}, \bibinfo{author}{C.~Zarges},
\newblock \bibinfo{title}{Computing longest common subsequences with the
  {B-Cell Algorithm}},
\newblock in: \bibinfo{booktitle}{Proc. of ICARIS 2012}, \bibinfo{year}{2012},
  pp. \bibinfo{pages}{111--124}.
\bibitem[{Cutello et~al.(2004)Cutello, Nicosia, and Pavone}]{optia}
\bibinfo{author}{V.~Cutello}, \bibinfo{author}{G.~Nicosia},
  \bibinfo{author}{M.~Pavone},
\newblock \bibinfo{title}{Exploring the capability of immune algorithms: A
  characterization of hypermutation operators},
\newblock in: \bibinfo{booktitle}{Proc. of ICARIS 2004}, \bibinfo{year}{2004},
  pp. \bibinfo{pages}{263--276}.
\bibitem[{Cutello et~al.(2003)Cutello, Nicosia, and Pavone}]{CutelloGecco03}
\bibinfo{author}{V.~Cutello}, \bibinfo{author}{G.~Nicosia},
  \bibinfo{author}{M.~Pavone},
\newblock \bibinfo{title}{A hybrid immune algorithm with information gain for
  the graph coloring problem},
\newblock in: \bibinfo{booktitle}{Proc. of GECCO 2003}, \bibinfo{year}{2003},
  pp. \bibinfo{pages}{171--182}.
\bibitem[{Cutello and Nicosia(2006)}]{cutello2006}
\bibinfo{author}{V.~Cutello}, \bibinfo{author}{G.~Nicosia},
\newblock \bibinfo{title}{A clonal selection algorithm for coloring, hitting
  set and satisfiability problems},
\newblock \bibinfo{journal}{Neural Nets} \bibinfo{volume}{3931}
  (\bibinfo{year}{2006}) \bibinfo{pages}{324--337}.
\bibitem[{Oliveto and Yao(2011)}]{OlivetoBookChapter}
\bibinfo{author}{P.~S. Oliveto}, \bibinfo{author}{X.~Yao},
\newblock \bibinfo{title}{Runtime analysis of evolutionary algorithms for
  discrete optimisation},
\newblock in: \bibinfo{booktitle}{Theory of Randomized Search Heuristics:
  Foundations and Recent Developments}, \bibinfo{publisher}{World Scientific},
  \bibinfo{year}{2011}, pp. \bibinfo{pages}{21--52}.
\bibitem[{Jansen(2013)}]{jansenbook}
\bibinfo{author}{T.~Jansen}, \bibinfo{title}{Analyzing Evolutionary Algorithms:
  The Computer Science Perspective}, \bibinfo{publisher}{Springer},
  \bibinfo{year}{2013}.
\bibitem[{Jansen and Zarges(2011)}]{Jansen2011}
\bibinfo{author}{T.~Jansen}, \bibinfo{author}{C.~Zarges},
\newblock \bibinfo{title}{Variation in artificial immune systems:
  Hypermutations with mutation potential},
\newblock in: \bibinfo{booktitle}{Proc. of ICARIS 2011}, \bibinfo{year}{2011},
  pp. \bibinfo{pages}{132--145}.
\bibitem[{Oliveto et~al.(2009)Oliveto, Lehre, and
  Neumann}]{OlivetoLehreNeumann2009}
\bibinfo{author}{P.~S. Oliveto}, \bibinfo{author}{P.~K. Lehre},
  \bibinfo{author}{F.~Neumann},
\newblock \bibinfo{title}{Theoretical analysis of rank-based mutation-combining
  exploration and exploitation},
\newblock in: \bibinfo{booktitle}{Proc. of CEC 2009}, \bibinfo{year}{2009}, pp.
  \bibinfo{pages}{1455--1462}.
\bibitem[{Corus and Oliveto(2018)}]{CorusOlivetoTEVC2018}
\bibinfo{author}{D.~Corus}, \bibinfo{author}{P.~S. Oliveto},
\newblock \bibinfo{title}{Standard steady state genetic algorithms can
  hillclimb faster than mutation-only evolutionary algorithms},
\newblock \bibinfo{journal}{IEEE Trans. Evol. Comput.} \bibinfo{volume}{22}
  (\bibinfo{year}{2018}) \bibinfo{pages}{720--732}.
\bibitem[{Lissovoi et~al.(2017)Lissovoi, Oliveto, and
  Warwicker}]{LissovoiOlivetoWarwickerGECCO2017}
\bibinfo{author}{A.~Lissovoi}, \bibinfo{author}{P.~S. Oliveto},
  \bibinfo{author}{J.~A. Warwicker},
\newblock \bibinfo{title}{On the runtime analysis of generalised selection
  hyper-heuristics for pseudo-boolean optimisation},
\newblock in: \bibinfo{booktitle}{Proc. of GECCO 2017}, \bibinfo{year}{2017},
  pp. \bibinfo{pages}{849--856}.
\bibitem[{Doerr et~al.(2018)Doerr, Lissovoi, Oliveto, and
  Warwicker}]{DoerrLissovoiOlivetoWarwickerGECCO2017}
\bibinfo{author}{B.~Doerr}, \bibinfo{author}{A.~Lissovoi},
  \bibinfo{author}{P.~S. Oliveto}, \bibinfo{author}{J.~A. Warwicker},
\newblock \bibinfo{title}{On the runtime analysis of selection hyper-heuristics
  with adaptive learning periods},
\newblock in: \bibinfo{booktitle}{Proc. of GECCO 2018}, \bibinfo{year}{2018},
  pp. \bibinfo{pages}{1015--1022}.
\bibitem[{Lengler(2018)}]{Lengler2018}
\bibinfo{author}{J.~Lengler},
\newblock \bibinfo{title}{A general dichotomy of evolutionary algorithms on
  monotone functions},
\newblock in: \bibinfo{booktitle}{Proc. of {PPSN} {XV}}, \bibinfo{year}{2018},
  pp. \bibinfo{pages}{3--15}.
\bibitem[{Friedrich et~al.(2018)Friedrich, G\"{o}bel, Quinzan, and
  Wagner}]{FriedrichPPSN2018}
\bibinfo{author}{T.~Friedrich}, \bibinfo{author}{A.~G\"{o}bel},
  \bibinfo{author}{F.~Quinzan}, \bibinfo{author}{M.~Wagner},
\newblock \bibinfo{title}{Heavy-tailed mutation operators in single-objective
  combinatorial optimization},
\newblock in: \bibinfo{booktitle}{Proc. of {PPSN} {XV}}, \bibinfo{year}{2018},
  pp. \bibinfo{pages}{134--145}.
\bibitem[{Droste et~al.(2002)Droste, Jansen, and Wegener}]{droste}
\bibinfo{author}{S.~Droste}, \bibinfo{author}{T.~Jansen},
  \bibinfo{author}{I.~Wegener},
\newblock \bibinfo{title}{On the analysis of the (1+ 1) evolutionary
  algorithm},
\newblock \bibinfo{journal}{Theoretical Computer Science} \bibinfo{volume}{276}
  (\bibinfo{year}{2002}) \bibinfo{pages}{51--81}.
\bibitem[{Dang et~al.(2018)Dang, Friedrich, K{\"o}tzing, Krejca, Lehre,
  Oliveto, Sudholt, and Sutton}]{OlivetoTEVC2017}
\bibinfo{author}{D.-C. Dang}, \bibinfo{author}{T.~Friedrich},
  \bibinfo{author}{T.~K{\"o}tzing}, \bibinfo{author}{M.~S. Krejca},
  \bibinfo{author}{P.~K. Lehre}, \bibinfo{author}{P.~S. Oliveto},
  \bibinfo{author}{D.~Sudholt}, \bibinfo{author}{A.~M. Sutton},
\newblock \bibinfo{title}{Escaping local optima using crossover with emergent
  diversity},
\newblock \bibinfo{journal}{IEEE Transactions on Evolutionary Computation}
  \bibinfo{volume}{22} (\bibinfo{year}{2018}) \bibinfo{pages}{484--497}.
\bibitem[{Doerr et~al.(2017)Doerr, Le, Makhmara, and Nguyen}]{DoerrGecco2017}
\bibinfo{author}{B.~Doerr}, \bibinfo{author}{H.~P. Le},
  \bibinfo{author}{R.~Makhmara}, \bibinfo{author}{T.~D. Nguyen},
\newblock \bibinfo{title}{Fast genetic algorithms},
\newblock in: \bibinfo{booktitle}{Proc. of GECCO 2017}, \bibinfo{year}{2017},
  pp. \bibinfo{pages}{777--784}.
\bibitem[{J{\"a}gersk{\"u}pper and Storch(2007)}]{JaegerskuepperStorch}
\bibinfo{author}{J.~J{\"a}gersk{\"u}pper}, \bibinfo{author}{T.~Storch},
\newblock \bibinfo{title}{When the plus strategy outperforms the comma strategy
  and when not},
\newblock in: \bibinfo{booktitle}{Proc. of FOCI 2007}, \bibinfo{year}{2007},
  pp. \bibinfo{pages}{25--32}.
\bibitem[{Oliveto et~al.(2017)Oliveto, Paix{\~a}o, P{\'e}rez~Heredia, Sudholt,
  and Trubenov{\'a}}]{Oliveto2017}
\bibinfo{author}{P.~S. Oliveto}, \bibinfo{author}{T.~Paix{\~a}o},
  \bibinfo{author}{J.~P{\'e}rez~Heredia}, \bibinfo{author}{D.~Sudholt},
  \bibinfo{author}{B.~Trubenov{\'a}},
\newblock \bibinfo{title}{How to escape local optima in black box optimisation:
  When non-elitism outperforms elitism},
\newblock \bibinfo{journal}{Algorithmica}  (\bibinfo{year}{2017}).
  \bibinfo{note}{{doi}: 10.1007/s00453-017-0369-2}.
\bibitem[{Paix{\~a}o et~al.(2017)Paix{\~a}o, P{\'e}rez~Heredia, Sudholt, and
  Trubenov{\'a}}]{Jorge2015}
\bibinfo{author}{T.~Paix{\~a}o}, \bibinfo{author}{J.~P{\'e}rez~Heredia},
  \bibinfo{author}{D.~Sudholt}, \bibinfo{author}{B.~Trubenov{\'a}},
\newblock \bibinfo{title}{Towards a runtime comparison of natural and
  artificial evolution},
\newblock \bibinfo{journal}{Algorithmica} \bibinfo{volume}{78}
  (\bibinfo{year}{2017}) \bibinfo{pages}{681--713}.
\bibitem[{Sudholt(2012)}]{Sudholt2012}
\bibinfo{author}{D.~Sudholt},
\newblock \bibinfo{title}{A new method for lower bounds on the running time of
  evolutionary algorithms},
\newblock \bibinfo{journal}{IEEE Transactions on Evolutionary Computation}
  \bibinfo{volume}{17} (\bibinfo{year}{2012}) \bibinfo{pages}{418--435}.
\bibitem[{Lissovoi et~al.(2019)Lissovoi, Oliveto, and
  Warwicker}]{LissovoiOlivetoWarwicker2019}
\bibinfo{author}{A.~Lissovoi}, \bibinfo{author}{P.~S. Oliveto},
  \bibinfo{author}{J.~A. Warwicker},
\newblock \bibinfo{title}{On the time complexity of algorithm selection
  hyper-heuristics for multimodal optimisation},
\newblock in: \bibinfo{booktitle}{Proc. of AAAI 2019}, \bibinfo{year}{2019}.
  \bibinfo{note}{To appear}.
\bibitem[{Feller(1968)}]{feller1968}
\bibinfo{author}{W.~Feller}, \bibinfo{title}{An Introduction to Probability
  Theory and Its Applications}, \bibinfo{publisher}{John Wiley \& Sons},
  \bibinfo{year}{1968}.
\bibitem[{Neumann et~al.(2008)Neumann, Sudholt, and Witt}]{NeuSudWit2008}
\bibinfo{author}{F.~Neumann}, \bibinfo{author}{D.~Sudholt},
  \bibinfo{author}{C.~Witt},
\newblock \bibinfo{title}{Rigorous analyses for the combination of ant colony
  optimization and local search},
\newblock in: \bibinfo{booktitle}{Proc. of {ANTS} 2008}, volume
  \bibinfo{volume}{5217}, \bibinfo{year}{2008}, pp. \bibinfo{pages}{132--143}.
\bibitem[{Corus et~al.(2018{\natexlab{a}})Corus, Oliveto, and
  Yazdani}]{CorusOlivetoYazdaniFASTAIS20018}
\bibinfo{author}{D.~Corus}, \bibinfo{author}{P.~S. Oliveto},
  \bibinfo{author}{D.~Yazdani},
\newblock \bibinfo{title}{Fast artificial immune systems},
\newblock in: \bibinfo{booktitle}{Proc. of PPSN XV},
  \bibinfo{year}{2018}{\natexlab{a}}, pp. \bibinfo{pages}{67--78}.
\bibitem[{Corus et~al.(2018{\natexlab{b}})Corus, Oliveto, and
  Yazdani}]{CorusOlivetoYazdaniPartition2018}
\bibinfo{author}{D.~Corus}, \bibinfo{author}{P.~S. Oliveto},
  \bibinfo{author}{D.~Yazdani},
\newblock \bibinfo{title}{Artificial immune systems can find arbitrarily good
  approximations for the {NP-hard} partition problem},
\newblock in: \bibinfo{booktitle}{Proc. of {PPSN} {XV}},
  \bibinfo{year}{2018}{\natexlab{b}}, pp. \bibinfo{pages}{16--28}.

\end{thebibliography}

\end{document}